\documentclass{article}

\usepackage[preprint, nonatbib]{neurips_2026}

\usepackage[utf8]{inputenc} 
\usepackage[T1]{fontenc}    
\usepackage{url}            
\usepackage{booktabs}       
\usepackage{amsfonts}       
\usepackage{nicefrac}       
\usepackage{microtype}      

\usepackage{graphicx}
\usepackage{subcaption}
\usepackage{adjustbox}
\usepackage{booktabs} 
\usepackage{multirow, threeparttable}
\usepackage{wrapfig}
\usepackage{caption}
\usepackage{hhline}
\usepackage{enumitem}
\usepackage{tikz}
\usepackage{tcolorbox}

\usepackage{algorithm2e}
\usepackage{bbm}
\RestyleAlgo{ruled}

\usepackage{pifont}
\usepackage[unicode=true,psdextra]{hyperref}

\usepackage{amsmath}
\usepackage{amssymb}
\usepackage{mathtools}
\usepackage{amsthm}
\usepackage{siunitx}
\usepackage{xspace}
\usepackage{tabularx}
\usepackage[square,numbers]{natbib}
\theoremstyle{plain}
\newtheorem{theorem}{Theorem}[section]
\newtheorem{proposition}[theorem]{Proposition}

\newtheorem{corollary}[theorem]{Corollary}
\theoremstyle{definition}

\theoremstyle{remark}

\newcommand{\diff}{\mathop{}\!\mathrm{d}}

\definecolor{lightblue}{HTML}{DAE8FC}
\definecolor{darkblue}{HTML}{6C8EBF}
\definecolor{orange}{HTML}{fdae6b}
\definecolor{darkorange}{HTML}{8c2d04}

\DeclareRobustCommand\circled[1]{\tikz[baseline=(char.base)]{\node[shape=circle,draw=darkblue,minimum size=0.35cm,inner sep=0pt,fill=lightblue] (char) {\fontfamily{phv}\selectfont \scriptsize \textbf{#1}};}}

\DeclareRobustCommand\bigcircle[1]{\tikz[baseline=(char.base)]{\node[shape=circle,draw=darkblue,minimum size=0.45cm,inner sep=0pt,fill=lightblue] (char) {\fontfamily{phv}\selectfont \normalsize \textbf{#1}};}}

\newcommand{\rqtag}[2]{%
\tikz[baseline=(rq.base)]\node[
    inner sep=2.0pt,
    rounded corners=2.5pt,
    draw=#1,
    line width=0.8pt,
    fill=#1!12,
    text=#1!85!black,
    font=\bfseries\footnotesize
](rq){RQ#2};%
}

\definecolor{rq1}{HTML}{4A69BD} 
\definecolor{rq2}{HTML}{60A3BC} 
\definecolor{rq3}{HTML}{6A89CC} 
\definecolor{rq4}{HTML}{8E44AD} 
\definecolor{rq5}{HTML}{F6B93B} 
\definecolor{rq6}{HTML}{E58E26} 
\definecolor{rq7}{HTML}{B8E994} 
\definecolor{rq8}{HTML}{82CCDD} 
\definecolor{rq9}{HTML}{636E72} 

\newcommand{\framework}{\textsc{OrthoBO}\xspace}

\newtcolorbox{mybox}{colback=white,colframe=darkblue}

\newcommand*\samethanks[1][\value{footnote}]{\footnotemark[#1]}

\title{\framework: Orthogonal Bayesian Hyperparameter Optimization}

\author{%
  Maresa Schröder\thanks{Equal contribution} \\
  LMU Munich\\
  Munich Center for Machine Learning\\
  \texttt{maresa.schroeder@lmu.de} 
  \And
  Pascal Janetzky\samethanks{} \\
  Bosch Center for Artificial Intelligence \\
  Munich Center for Machine Learning\\
  \texttt{pascal.janetzky@bosch.com} 
  \AND
  Michael Klar \\
  Bosch Center for Artificial Intelligence \\
  \texttt{michael.klar@bosch.com} \\
  \And
  Stefan Feuerriegel \\
  LMU Munich\\
  Munich Center for Machine Learning\\
  \texttt{feuerriegel@lmu.de} 
}

\begin{document}

\maketitle


\begin{abstract}
Bayesian optimization is widely used for hyperparameter optimization when model evaluations are expensive; however, noisy acquisition estimates can lead to unstable decisions. We identify acquisition estimation noise as a failure mode that was previously overlooked: even when the surrogate model and acquisition target are correctly specified, finite-sample Monte Carlo error can perturb acquisition values. This can, in turn, flip candidate rankings and lead to suboptimal BO decisions. As a remedy, we aim at variance reduction and propose an \emph{orthogonal acquisition estimator} that subtracts an optimally weighted score-function control variate, which yields an acquisition residual \emph{orthogonal} to posterior score directions and which thus reduces Monte Carlo variance. We further introduce \framework: a Bayesian optimization framework that combines our orthogonal acquisition estimator with ensemble surrogates and an outer log transformation. We show theoretically that our estimator preserves the target, leads to variance reduction, and improves pairwise ranking stability. We further verify the theoretical properties of \framework through numerical experiments where our framework reduces acquisition estimation variance, stabilizes candidate rankings, and achieves strong performance. We also demonstrate the downstream utility of \framework in hyperparameter optimization for neural network training and fine-tuning. 
\end{abstract}


\section{Introduction}

Hyperparameter optimization (HPO) is a fundamental component of modern machine learning systems, and has been used across a wide range of domains, including biology \citep{ban2017efficient,quitadadmo2017bayesian}, chemistry \citep{yuan2021systematic}, manufacturing \citep{albahar2021robust}, medicine \citep{nath2021power}, and physics \citep{tani2024comparison}. Because evaluating hyperparameter candidates can require substantial computation, \emph{Bayesian optimization (BO)} is frequently used to identify optimal hyperparameter configurations under limited evaluation budgets (e.g., \cite[][]{Feurer.2015, Shahriari.2016, Snoek.2012, Turner.2020, Wu.2020}). BO fits a probabilistic surrogate model of the objective and selects new configurations by maximizing an \emph{acquisition function} such as expected improvement (EI) \citep{Jones.1998}.\footnote{For a thorough introduction to BO, see, e.g., \cite[][]{Frazier.2018, Shahriari.2016}.} 

In our work, we focus on a source of instability that is central to BO decisions but has received comparatively little attention: \emph{the estimation of the acquisition value itself}. Acquisition functions are often treated as deterministic once a surrogate model has been fitted. In many practical BO pipelines, however, the acquisition value is obtained by marginalizing over uncertain surrogate parameters, aggregating across surrogate models, or using Monte Carlo (MC) approximations (e.g., \cite[][]{Bodin.2020, Hernandez.2015,Polyzos.2023, Snoek.2012}). As a result, the quantity used to rank candidate configurations is itself a noisy estimate. This may hurt the performance of BO because BO acts on \emph{rankings}: even small estimation errors can flip the order of two candidates with similar acquisition values and lead the BO algorithm to evaluate many suboptimal configurations (see Fig.~\ref{fig:motivation}).

At a technical level, we treat the estimation of the acquisition function as a MC estimation problem \cite[e.g.,][]{Bodin.2020, Hernandez.2014,Hernandez.2015,Snoek.2012} and propose a novel \emph{orthogonalized estimator of the acquisition value} that serves as a variance-reduction technique, thus improving robustness to sampling errors and stabilizing acquisition rankings. Our notion of orthogonality stems from the orthogonal ML literature (e.g., \cite[][]{Chernozhukov.2018, Foster.2023, Kennedy.2024, Mackey.2018}) in that our estimator offers robustness due to \emph{variance reduction} guarantees Note that our notion of robustness differs from that in existing `robust' regression methods, which aim to protect the surrogate model against corrupted or heavy-tailed observations (e.g., \cite[][]{Martinez.2018}).
\begin{wrapfigure}[21]{r}{0.45\textwidth} 
    \vspace{-0.3cm}
    \centering
    \includegraphics[width=1\linewidth]{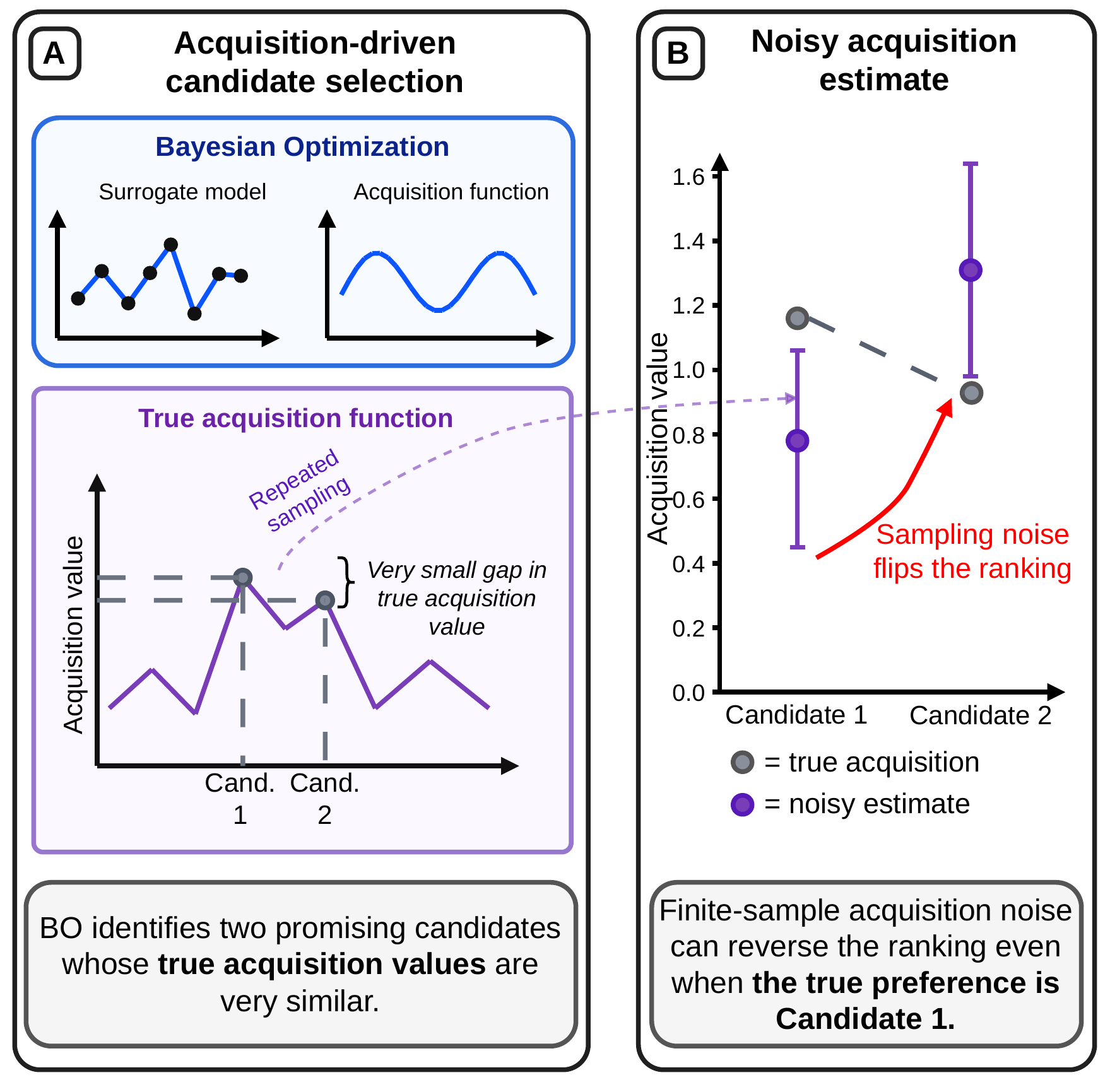}
    \vspace{-0.6cm}
    \caption{\textbf{Failure mode in Bayesian HPO:} estimation noise in the acquisition value can flip rankings and lead to suboptimal configurations.}
    \label{fig:motivation}
\end{wrapfigure}

\vspace{-0.4cm}
\textbf{Example.} Consider tuning a classification model on a large-scale image dataset, where each evaluation requires training on many images and assessing performance under an imbalanced class distribution. Two hyperparameter configurations may differ only slightly (e.g., a deeper network with stronger dropout versus a shallower network with weaker regularization). If estimates of the acquisition values are subject to noise, BO may spend a suboptimal configuration simply because it was ranked first due to sampling noise.

Our contribution is complementary to existing work on robust BO. Prior methods have mainly focused on improving the surrogate model, for example by accounting for kernel misspecification, hyperparameter uncertainty, unreliable posterior uncertainty, trust-region localization, or ensemble surrogates \citep{Berkenkamp.2019,Bogunovic.2021,Eriksson.2019,Lu.2023,Neiswanger.2021,Polyzos.2023,Turner.2020}. Other work has improved the acquisition formula or its numerical optimization \citep{Ament.2023}.
However, these approaches do \textbf{not} directly address the \emph{variance of the estimated acquisition values} that drive BO decisions.
In contrast, our method improves the statistical quality of the estimated acquisition values used for candidate selection through a novel variance reduction technique based on orthogonal scores.

\textbf{\framework:} We build on our orthogonal acquisition estimator and introduce \framework, a BO framework for robust hyperparameter optimization. \framework uses our orthogonalized acquisition estimation as a core component and combines it with ensemble surrogates \cite{Lu.2023, Polyzos.2023} and an outer log transformation \citep{Ament.2023}. Here, each component serves a distinct role: our orthogonalization stabilizes acquisition estimation, ensembling addresses structural misspecification and the outer log transformation improves numerical conditioning during acquisition maximization. 

Our \framework offers several theoretical guarantees. (i)~Our proposed orthogonalized acquisition estimator is unbiased for the marginal EI target and achieves lower variance than the corresponding na{\"i}ve MC estimator ($\rightarrow$ Theorem~\ref{thm:variance_reduction}). (ii)~We further show that the lower acquisition estimation variance improves the stability of pairwise acquisition rankings ($\rightarrow$ Proposition~\ref{prop:ranking}). Intuitively, this reduces the probability of falsely selecting a candidate based on MC noise, and leads thus to performance improvements in practice. 

Our \textbf{main contributions} are: \textbf{(1)}~We propose a novel, orthogonalized estimator for hyperparameter acquisition, reducing estimation variance without changing the underlying acquisition target. \textbf{(2)}~We develop \framework, a BO framework that explicitly targets acquisition estimation noise, structural surrogate misspecification, and numerical instabilities. \textbf{(3)}~We provide theoretical guarantees for variance reduction and ranking stability, and demonstrate improved empirical performance on various HPO benchmarks. Importantly, \framework is broadly applicable to essentially all BO problems.


\section{Related work\protect\footnote{We provide an extended related work in Supplement~\ref{sec:appendix_related_work}. Therein, we also provide a broader review of BO under different types of misspecification and orthogonal ML.}}
\label{sec:related_work}

\textbf{BO for hyperparameter tuning.}
BO is a standard approach for HPO (e.g., \cite[][]{Feurer.2015, Snoek.2012, Wu.2020}), particularly when evaluations of the objective function are expensive. In practice, BO relies on a surrogate model to approximate the objective and guide the search. A common choice for the surrogate model are Gaussian processes (GPs) \cite{Ha.2023, Jones.1998, Snoek.2012}; other alternatives include tree-based and density-based surrogates such as tree-structured Parzen estimators (TPE) \cite{Bergstra.2011, Bergstra.2013, Watanabe.2023}. Subsequent work has addressed scaling challenges through multi-fidelity methods, batch BO, and trust-region approaches (e.g., \cite[][]{Eriksson.2019, Falkner.2018, Gonzalez.2016, Wilson.2018}). \emph{However, no method has explicitly studied or addressed acquisition estimation noise as a failure mode.}

\textbf{Acquisition functions.}
Surrogate models are used by acquisition functions, which are heuristics to evaluate the utility of hyperparameter candidates before querying the (expensive) objective function \citep{Balandat.2020}. The performance of the acquisition function thus directly influences the overall performance. In the literature, several types of acquisition functions or acquisition frameworks have been proposed \cite{Ament.2023,Ament.2024,Eriksson.2019,moss2021gibbon,wang2017max,wilson2017reparameterization}. Prominent examples include, e.g., qLogEI \citep{Ament.2023}, TuRBO \citep{Eriksson.2019}, and UCB \citep{wilson2017reparameterization}. \emph{Our work is orthogonal to these works, as we focus on improving the \textbf{estimation} and \textbf{not} the acquisition function itself.\footnote{We provide a generalization of \framework at the acquisition function level in Supplement~\ref{sec:appendix_constrained_parallel}.}}

\textbf{Robustness to surrogate misspecification.}
A central weakness of BO is that it relies on a surrogate model which might be misspecified. Prior work has studied BO under misspecified kernel classes \cite{Bogunovic.2021}, unknown hyperparameters \cite{Berkenkamp.2019, Ha.2023}, unreliable posterior uncertainty \cite{Neiswanger.2021}, and a surrogate complexity-efficiency trade-off \cite{Bodin.2020}. Several practical strategies have been developed to mitigate such failure modes \citep{Ament.2024,Eriksson.2019,Lu.2023,Martinez.2018,Polyzos.2023}
\emph{However, these approaches address {structural} surrogate misspecification  and thus model uncertainty. In contrast, our method additionally targets estimation noise.}

\textbf{Research gap.} We identify acquisition estimation errors as a crucial failure mode in BO that has been previously overlooked. To address this, we propose an orthogonalized estimator aimed at variance reduction and that stabilizes acquisition-based decisions. To the best of our knowledge, this is the first work to explicitly formulate acquisition value estimation in BO as a variance reduction problem and to develop an orthogonalized estimator for improving acquisition-based decisions.


\section{Problem setting}
\label{sec:setting}

\textbf{Setup.}
We consider the standard setup for HPO over a search space $\Lambda \subseteq \mathbb{R}^d$ \citep{Ament.2024, eriksson2021scalable}. For a hyperparameter configuration $\lambda \in \Lambda$, let $f(\lambda)$ denote the unknown performance of interest, such as, e.g., the validation loss or validation error after training a model with hyperparameters $\lambda$. The objective $f$ is expensive to evaluate and can only be observed through noisy evaluations $y = f(\lambda) + \epsilon$, where $\epsilon$ captures evaluation noise. At each iteration $t = 1,\dots,T$, the available data are $\mathcal{D}_t = \{(\lambda_i, y_i)\}_{i=1}^t$. We aim to sequentially choose hyperparameter configurations $\lambda_1, \lambda_2, \dots, \lambda_T$ to identify a high-performing configuration within an evaluation budget $T$.

\textbf{Recap: BO.}
The central idea of BO is to maintain a probabilistic model of the unknown objective and to use it to guide future evaluations \citep{frazier2018tutorial}.  At iteration $t =1,\ldots,T$, BO fits a probabilistic \emph{surrogate model} $m$ to the current dataset and uses an \textit{acquisition function} $\alpha$ to select the next evaluation point \citep{Ament.2024}. The surrogate provides two crucial quantities for decision-making: (i)~a prediction of objective values in unexplored regions, and (ii)~a measure of uncertainty about those predictions. BO then combines these two ingredients through $\alpha$, which scores candidate points according to their potential utility for optimization. After selecting a new candidate point, the objective $f$ is evaluated, and the new observation is added to the dataset $\mathcal{D}$. The surrogate is then updated, and the process repeats. In this way, BO adaptively concentrates evaluations in regions that are both informative and promising. As a result, the performance of BO depends on two components: the quality of the surrogate model and the stability of the acquisition values used to rank candidate configurations.

\textbf{Surrogate model.}
Let $m$ index a surrogate model, and $\theta_m$ denote its latent parameters. Conditional on $\theta_m$, the surrogate induces a posterior predictive distribution for the objective $f(\lambda)$ given $\mathcal{D}_t$. In practice, the parameters $\theta_m$ are not known and must themselves be inferred from data. 

In our work, we focus on \emph{expected improvement (EI)} \citep{Jones.1998} as our acquisition function. For a surrogate model $m$ with parameters $\theta_m$, we define
\begin{equation}\label{eq:EI}
\mathrm{EI}_m(\lambda_t; \theta_m)
= \mathbb{E}\big[(f^\star - f(\lambda_t))_+ \mid \mathcal{D}_t, \theta_m\big],
\end{equation}
where $f^\star = \min_{i \le t} y_i$ is the currently best observed value up to iteration $t$ and $(x)_{+}:=\max\{0,x\}$ \cite{Jones.1998,Snoek.2012}. In practice, the model parameters are not fixed but follow an approximate posterior $\theta_m \sim q_{m,t}(\theta)$.
The marginal EI over the parameters is then given by
\begin{equation}\label{eq:marg_EI}
\mathrm{EI}_m^{\mathrm{marg}}(\lambda_t)
= \mathbb{E}_{\theta_m \sim q_{m,t}}[\mathrm{EI}_m(\lambda_t; \theta_m)],
\end{equation}
which is commonly approximated through MC samples $s=1,\ldots,S$ given a specific MC sampling budget $S$ \cite{Balandat.2020,Snoek.2012, Wilson.2018}, i.e.,
\vspace{-0.2cm}
\begin{equation}\label{equation:ei_mc}
\widehat{\mathrm{EI}}_m^{\mathrm{MC}}(\lambda_t) = \frac{1}{S} \sum_{s=1}^S \mathrm{EI}_m(\lambda_t; \theta_m^{(s)}), \quad \theta_m^{(s)} \sim q_{m,t}.
\end{equation}

\vspace{-0.4cm}
However, $\widehat{\mathrm{EI}}_m^{\mathrm{MC}}$ suffers from high variance for small budgets $S$ and sensitivity to approximation errors in the posterior $q_{m,t}$, which leads to instability in the estimated acquisition and thus suboptimal BO decisions.
 
\textbf{Finite-sample acquistion estimation noise.}\footnote{We discuss structural surrogate misspecification due to improper surrogate model classes and numerical instabilities due to small or skewed acquisition values in Supplement~\ref{sec:appendix_bo_failure_modes}.} Bayesian HPO performance depends on the stability of acquisition values used to rank candidate configurations. In practice, these values are often obtained by marginalizing over surrogate uncertainty and must be approximated from a finite number of samples. As a result, the acquisition function itself becomes a noisy estimate: even when the surrogate model is reasonable, finite-sample MC error can perturb acquisition values, flip candidate rankings, and thus lead BO to evaluate suboptimal configurations. We identify this \emph{acquisition estimation noise} as a overlooked but practically important failure mode.

\textbf{Our work.} 
In this work, we address instability in BO at the level of the acquisition estimate itself. Even with a reasonable surrogate family, BO may still perform poorly if the acquisition estimator used to rank candidate points is noisy, sensitive to approximate inference, or numerically ill-conditioned. 

We treat the EI-scale marginal acquisition obtained by integrating over surrogate uncertainty as an MC estimation problem. We then propose an orthogonal acquisition estimator that reduces variance induced by surrogate uncertainty, while preserving the underlying EI-scale target. Therein, we subtract an optimally weighted score-function control variate, which yields an acquisition residual \emph{orthogonal} to posterior score directions and which thus reduces MC variance. In \framework, we further combine our orthogonality acquisition estimator for variance reduction with an ensemble of surrogate models and adaptive model weighting optimized through an outer log transformation to obtain a full, state-of-the-art BO framework. 

\vspace{-0.2cm}
\section{Orthogonal Bayesian optimization}
\label{sec:method}
\vspace{-0.2cm}

\begin{figure}[h]
    \centering
    \includegraphics[width=1\linewidth]{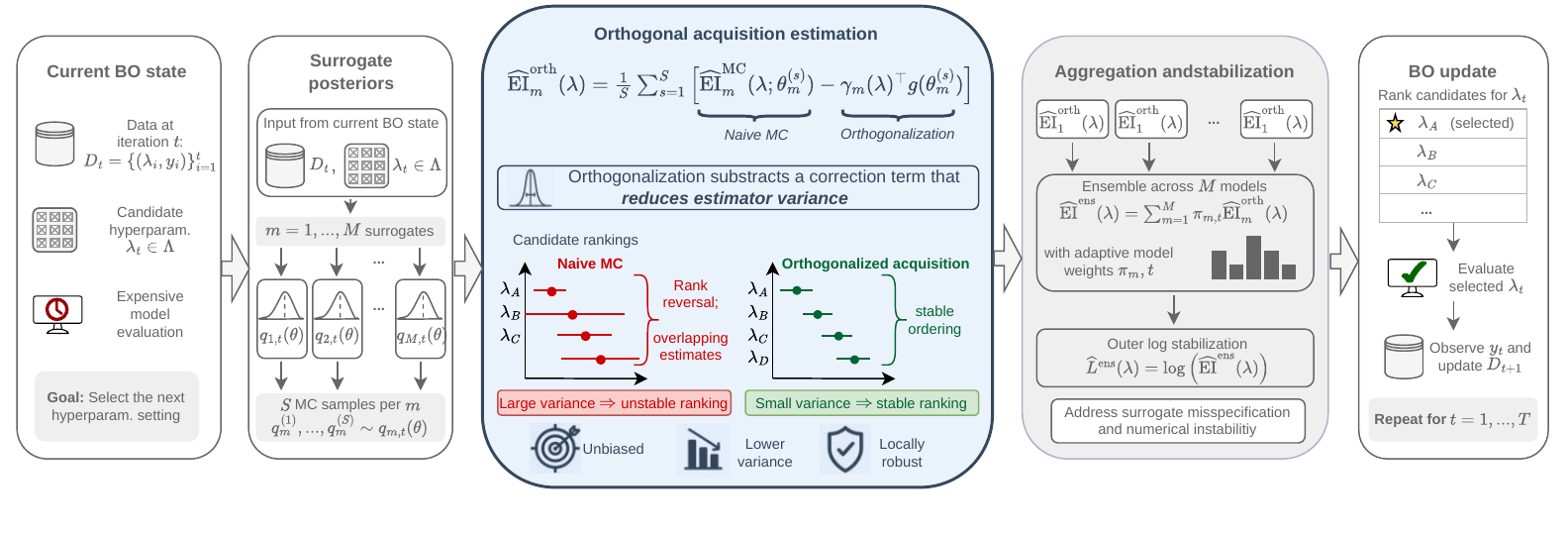}
    \vspace{-.7cm}
    \caption{\textbf{Our proposed \framework.} It improves acquisition estimation through \emph{variance reduction}. For this, we orthogonalize the acquisition estimation \emph{without} changing the acquisition target.}
    \label{fig:framework}
    \vspace{-0.2cm}
\end{figure}

Below, we present \framework, our orthogonalized framework for BO that addresses the failure modes discussed above (an overview is shown in Fig.~\ref{fig:framework}). We first present our \emph{main contribution}: the \emph{orthogonal acquisition estimator} (Section~\ref{sec:orthogonal_acquisition_estimator}). Of note, we do \emph{not} change the acquisition target, but rather improve the estimation through \emph{variance reduction} techniques. Then, we provide example instantiations of our orthogonal framework for two common surrogate classes (Section~\ref{sec:surrogate_instantiations}). 
We then present \framework, where we embed our orthogonal acquisition estimator inside a state-of-the-art BO algorithm (Section~\ref{sec:optimization_procedure}).  
In the main text, we focus on EI-based acquisition functions in order to keep the notation and theoretical arguments simple. The key object throughout is the marginal EI under surrogate uncertainty and ensemble aggregation. We provide extensions to constrained and parallel acquisition functions in Supplement~\ref{sec:appendix_constrained_parallel}.

\subsection{Orthogonal acquisition estimator}
\label{sec:orthogonal_acquisition_estimator}

We now address the instability due to acquisition estimation noise. Our key idea is to introduce an \emph{orthogonalized acquisition functional} that shares the same marginal target as the original EI but has reduced variance under $q_{m,t}$. Let $g(\theta) := \nabla_\theta \log q_{m,t}(\theta)$ denote the posterior score function.
\begin{mybox}
\vspace{-0.2cm}
We define the \emph{orthogonalized EI} as
\vspace{-0.1cm}
\begin{equation}\label{eq:ortho_EI_functional}
    \mathrm{EI}_m^{\mathrm{orth}}(\lambda; \theta)
    \;:=\; \mathrm{EI}_m(\lambda; \theta)
    - \gamma_m(\lambda)^\top g(\theta),
\end{equation}
where $\gamma_m(\lambda) := \Sigma_g^{-1}\,\operatorname{Cov}_{q_{m,t}}\!\big(g,\,\mathrm{EI}_m(\lambda;\cdot)\big)$ and $\Sigma_g := \operatorname{Cov}_{q_{m,t}}(g, g)$. The corresponding marginal acquisition is
\vspace{-0.2cm}
\begin{equation}\label{eq:ortho_EI_marg}
    \mathrm{EI}_m^{\mathrm{orth, marg}}(\lambda)
    \;:=\; \mathbb{E}_{\theta \sim q_{m,t}}\!\left[\mathrm{EI}_m^{\mathrm{orth}}(\lambda; \theta)\right].
\end{equation}
The associated MC estimator is \vspace{-0.2cm}
\begin{equation}\label{eq:ortho_EI}
    \widehat{\mathrm{EI}}_m^{\mathrm{orth}}(\lambda)
    = \frac{1}{S}\sum_{s=1}^S \mathrm{EI}_m^{\mathrm{orth}}(\lambda; \theta_m^{(s)}),
    \qquad \theta_m^{(s)} \sim q_{m,t}.
\end{equation}
\vspace{-0.5cm}
\end{mybox}
Of note, the construction is non-trivial because $g(\theta)$ is based on the \emph{parameter posterior}, whereas common MC pipelines sample only from the predictive posterior over function values. We therefore draw parameter samples explicitly and evaluate both $\mathrm{EI}_m$ and $g$ at the same $\theta_m^{(s)}$. In practice, $\gamma_m(\lambda)$ is unknown and we use the empirical plug-in $\hat\gamma_m(\lambda)$ estimated from the same MC samples; this introduces an in-sample bias of order $O(1/S)$ that is dominated by the $O(1/\sqrt{S})$ sampling standard deviation at the MC budgets used in our experiments.\footnote{Sample-splitting (cross-fitting) would restore exact finite-sample unbiasedness without changing the leading asymptotic variance; see Supplement~\ref{sec:appendix_proofs}.}

\textbf{Properties of the orthogonalized acquisition.} The orthogonalized EI satisfies three key properties: (i)~\emph{preserved target}: it has the same marginal as the original EI, so the BO target is unchanged; (ii)~\emph{variance reduction}: it has smaller variance under $q_{m,t}$; and (iii)~\emph{local robustness}: it is first-order insensitive to score-tilt perturbations of $q_{m,t}$.

\textbf{Regularity conditions for the score function.}
Throughout our analysis, we assume that the surrogate parameter distribution $q_{m,t}$ has a differentiable density on a support $\Theta \subseteq \mathbb{R}^{d_\theta}$ and that the corresponding boundary term vanishes, i.e., $\int_{\Theta} \nabla_\theta q_{m,t}(\theta)\, \diff \theta = 0.$
This holds, for example, if $\Theta=\mathbb{R}^{d_\theta}$ and $q_{m,t}(\theta)\to 0$ sufficiently fast as $\|\theta\|\to\infty$.

\begin{theorem}[Variance reduction]
\label{thm:variance_reduction}
Assume that $\mathbb{E}_{q_{m,t}}[\mathrm{EI}_m(\lambda;\theta)^2] < \infty$ and $\mathbb{E}_{q_{m,t}}[\|g(\theta)\|_2^2] < \infty$, that $\Sigma_g$ is nonsingular, and that the standard score-function regularity conditions hold for $q_{m,t}$. Then:

\textbf{(i) Preserved target.}
$\mathbb{E}[\mathrm{EI}_m^{\mathrm{orth}}(\lambda; \theta)] = \mathrm{EI}_m^{\mathrm{marg}}(\lambda; \theta)$.

\textbf{(ii) Variance reduction.} Under $q_{m,t}$, we achieve
\begin{equation}
    \operatorname{Var}\!\big(\mathrm{EI}_m^{\mathrm{orth}}(\lambda; \theta)\big)
    = \operatorname{Var}\!\big(\mathrm{EI}_m(\lambda; \theta)\big)
    - \operatorname{Cov}(g, \mathrm{EI}_m)^\top \Sigma_g^{-1} \operatorname{Cov}(g, \mathrm{EI}_m) \leq \operatorname{Var}\!\big(\mathrm{EI}_m(\lambda; \theta)\big).
\end{equation}
\end{theorem}
\vspace{-0.4cm}
\begin{proof}
We provide a proof in Supplement~\ref{sec:appendix_proofs}.
\end{proof}
\vspace{-0.2cm}
Due to the orthogonal score construction, our estimator is further locally robust to perturbations in the score-tilt direction.
\begin{corollary}[Local robustness.] 
\label{cor:robustness}
Assume the conditions of Theorem~\ref{thm:variance_reduction}. For any $b \in \mathbb{R}^{d_\theta}$, consider the tilted family $q_{m,t}^{(\varepsilon)}(\theta) \propto q_{m,t}(\theta)\exp\!\big(\varepsilon\, b^\top g(\theta)\big)$ for small $|\varepsilon|$. Holding $\gamma_m(\lambda)$ fixed at its $\varepsilon=0$ value, it holds that
\vspace{-0.2cm}
\begin{equation}
    \left.\frac{\diff}{\diff\varepsilon}\, \mathbb{E}_{q_{m,t}^{(\varepsilon)}}\!\left[\mathrm{EI}_m^{\mathrm{orth}}(\lambda; \theta)\right]
    \right|_{\varepsilon=0}
    \;=\; 0.
\end{equation}
\end{corollary}
\vspace{-0.4cm}
\begin{proof}
We provide a proof in Supplement~\ref{sec:appendix_proofs}.
\end{proof}

\textbf{Implications for acquisition estimation.} Theorem~\ref{thm:variance_reduction}(i) shows that orthogonalization preserves the BO target. Since the MC estimator is an i.i.d.\ sample mean, $\operatorname{Var}(\widehat{\mathrm{EI}}_m^{\mathrm{orth}}(\lambda)) = \operatorname{Var}(\mathrm{EI}_m^{\mathrm{orth}}(\lambda;\theta))/S$, Theorem~\ref{thm:variance_reduction}(ii) directly implies $\operatorname{Var}(\widehat{\mathrm{EI}}_m^{\mathrm{orth}}(\lambda)) \le \operatorname{Var}(\widehat{\mathrm{EI}}_m^{\mathrm{MC}}(\lambda))$ for any MC budget $S$, which improves the stability of acquisition-based decisions. Corollary~\ref{cor:robustness} provides local robustness in the orthogonal ML sense \citep{Chernozhukov.2018, Kennedy.2024}: small score-tilt perturbations of $q_{m,t}$ do not affect the orthogonalized acquisition to first order. Importantly, Theorem~\ref{thm:variance_reduction} is a \emph{within-model} result: orthogonalization stabilizes acquisition estimation conditional on $q_{m,t}$, but does not directly correct structural misspecification of the surrogate family. We address the latter through ensemble modeling in Section~\ref{sec:optimization_procedure}.

\begin{proposition}[Pairwise ranking stability]
\label{prop:ranking}
Let $\lambda, \lambda' \in \Lambda$ such that $\Delta(\lambda,\lambda') := \mathrm{EI}_m^{\mathrm{marg}}(\lambda) - \mathrm{EI}_m^{\mathrm{marg}}(\lambda') > 0$. Let $\widehat{\Delta}_{\mathrm{MC}}(\lambda,\lambda')$ and $\widehat{\Delta}_{\mathrm{orth}}(\lambda,\lambda')$ denote the corresponding Monte Carlo difference estimators. 
Then both estimators are unbiased for $\Delta(\lambda,\lambda')$, and 
\footnotesize
\begin{equation}
\mathbb{P}\!\left(\widehat{\Delta}_{\mathrm{orth}}(\lambda,\lambda') \le 0\right)
    \;\leq\; \frac{\operatorname{Var}\!\left(\widehat{\Delta}_{\mathrm{orth}}(\lambda,\lambda')\right)}{\operatorname{Var}\!\left(\widehat{\Delta}_{\mathrm{orth}}(\lambda,\lambda')\right) + \Delta(\lambda,\lambda')^2}
    \;\leq\; \frac{\operatorname{Var}\!\left(\widehat{\Delta}_{\mathrm{MC}}(\lambda,\lambda')\right)}{\operatorname{Var}\!\left(\widehat{\Delta}_{\mathrm{MC}}(\lambda,\lambda')\right) + \Delta(\lambda,\lambda')^2}.
\end{equation}
\end{proposition}
\normalsize
\vspace{-0.5cm}
\begin{proof}
We provide a proof for Proposition~\ref{prop:ranking} in Supplement~\ref{sec:appendix_proofs}.
\end{proof}
\vspace{-0.2cm}

Proposition~\ref{prop:ranking} shows how acquisition estimation noise propagates into BO decisions. High estimator variance can flip the estimated sign of the acquisition gap between two candidates points, thus causing BO to select a suboptimal point. Orthogonalization reduces the variance without changing the acquisition target, and therefore directly reduces the probability of such ranking errors.

\subsection{Instantiations for common surrogates}
\label{sec:surrogate_instantiations}

We instantiate \framework for two surrogate classes common in HPO: \circled{A} Gaussian process and \circled{B} tree-based density surrogates. The first corresponds to the standard BO setting, while the second captures widely used non-Gaussian and nonparametric approaches such as TPE-style optimization.

\vspace{-0.2cm}
\subsubsection*{\bigcircle{A} Gaussian process (GP) surrogate}
\label{sec:gp_instantiation}
\vspace{-0.2cm}
For GP surrogate models \cite{Ha.2023, Jones.1998, Snoek.2012}, acquisition estimation noise is often driven by sensitivity of the mean $\mu(\cdot)$ and variance $\sigma(\cdot)$ function to lengthscale and noise hyperparameters. We show that orthogonalization reduces MC fluctuations induced by uncertainty in these directions. This is especially useful early in optimization when the approximation of the posterior over the GP hyperparameters is diffuse.
Let
\begin{equation}
f \sim \mathcal{GP}\!\left(m_\theta(\cdot), k_\theta(\cdot,\cdot)\right),
\end{equation}
where here $m_\theta$ denotes the mean function (typically taken constant or zero) and $k_\theta$ the covariance kernel with $\theta \in \Theta \subset \mathbb{R}^{d_\theta}$ collecting the surrogate hyperparameters, such as kernel lengthscales, amplitude, and noise variance. Given data $\mathcal{D}_t = \{(\lambda_i,y_i)\}_{i=1}^t$, the predictive distribution at a candidate $\lambda$ is
\begin{equation}
f(\lambda)\mid \mathcal{D}_t,\theta
\sim
\mathcal{N}\!\left(\mu_t(\lambda;\theta), \sigma_t^2(\lambda;\theta)\right),
\end{equation} 
with
\begin{align}
\mu_t(\lambda;\theta)
&= m_\theta(\lambda) + k_\theta(\lambda,\Lambda_t)^\top
\left(K_\theta(\Lambda_t,\Lambda_t)+\sigma_n^2 I_t\right)^{-1}
\left(y_t-m_\theta(\Lambda_t)\right), \\
\sigma_t^2(\lambda;\theta)
&= k_\theta(\lambda,\lambda) - k_\theta(\lambda,\Lambda_t)^\top
\left(K_\theta(\Lambda_t,\Lambda_t)+\sigma_n^2 I_t\right)^{-1}
k_\theta(\lambda,\Lambda_t),
\end{align}
where $\Lambda_t=(\lambda_1,\dots,\lambda_t)$ and $y_t=(y_1,\dots,y_t)^\top$, and $K$ denotes the Gram matrix.
We aim to maximize the EI
\begin{equation}
\mathrm{EI}^{\mathrm{GP}}(\lambda;\theta)
= (f^\star-\mu_t(\lambda;\theta))\,\Phi(z_t(\lambda;\theta)) +
\sigma_t(\lambda;\theta)\,\phi(z_t(\lambda;\theta)),
\end{equation}
where
\begin{equation}
z_t(\lambda;\theta) = \frac{f^\star-\mu_t(\lambda;\theta)}{\sigma_t(\lambda;\theta)},
\qquad f^\star = \min_{i \le t} y_i,
\end{equation}
and $\Phi$ and $\phi$ denote the standard normal cumulative distribution and probability density function.

We assume an approximate posterior $q_t(\theta)$ over GP hyperparameters, obtained for example via a Laplace approximation or variational Gaussian approximation. The marginal acquisition is then \vspace{-0.1cm}
\begin{equation}
\mathrm{EI}^{\mathrm{GP,marg}}(\lambda) = \mathbb{E}_{\theta \sim q_t}\!\left[\mathrm{EI}^{\mathrm{GP}}(\lambda;\theta)\right],
\end{equation}
and the orthogonalized estimator is
\begin{mybox}
\vspace{-0.3cm}
\begin{equation}
\widehat{\mathrm{EI}}^{\mathrm{GP,orth}}(\lambda) = \frac{1}{S}\sum_{s=1}^S \left[ \mathrm{EI}^{\mathrm{GP}}(\lambda;\theta^{(s)}) - \gamma^{\mathrm{GP}}(\lambda)^\top g_t(\theta^{(s)})
\right], \qquad \theta^{(s)} \sim q_t,
\end{equation}
\vspace{-0.38cm}
\end{mybox}
where $g_t(\theta) = \nabla_\theta \log q_t(\theta)$ and $\gamma^{\mathrm{GP}}(\lambda) = \operatorname{Cov} (g_t,g_t)^{-1}\operatorname{Cov}\!\left(g_t,\mathrm{EI}^{\mathrm{GP}}(\lambda;\theta)\right)$.

\subsubsection*{\bigcircle{B} Tree-based / TPE-style surrogate}
\label{sec:tpe_instantiation}
\vspace{-0.2cm}
As a second instantiation, we consider tree-based or density-based surrogates as in Tree-structured Parzen Estimators (TPE) \cite{Bergstra.2011, Bergstra.2013, Watanabe.2023}. Rather than modeling $p(y \mid \lambda)$ directly, TPE-style methods model \vspace{-0.1cm}
\begin{equation}
p(\lambda \mid y)
= \begin{cases}
\ell_\theta(\lambda), & y < y^\star, \\
g_\theta(\lambda), & y \ge y^\star,
\end{cases}
\end{equation}
where $y^\star$ is a quantile threshold and $\theta$ denotes model parameters, such as kernel density bandwidths, mixture weights, or tree parameters. In this setting, the acquisition is proportional to a density ratio: \vspace{-0.1cm}
\begin{equation}
\mathrm{EI}^{\mathrm{TPE}}(\lambda;\theta)
\propto
\frac{\ell_\theta(\lambda)}{g_\theta(\lambda)}.
\end{equation}
More generally, we denote the EI as $h^{\mathrm{TPE}}(\lambda;\theta) := \mathrm{EI}^{\mathrm{TPE}}(\lambda;\theta)$, where $h^{\mathrm{TPE}}$ represents either the exact acquisition induced by the density model or a practical approximation.

Unlike the GP case, differentiability of $q_t(\theta)$ may not be possible. We therefore consider a more general control variate construction. Let $c_t(\theta) \in \mathbb{R}^{d_c}$ denote any random vector satisfying $\mathbb{E}_{q_t}[c_t(\theta)] = 0$, such as centered bootstrap statistics, centered density ratio proxies, centered bandwidth or split summaries, or score functions when they exist. We then obtain our orthogonal acquisition estimator by
\vspace{-0.58cm}

\begin{mybox}
\vspace{-0.55cm}   
\begin{equation}
\widehat{\mathrm{EI}}^{\mathrm{TPE,orth}}(\lambda) = \frac{1}{S}\sum_{s=1}^S \left[ h^{\mathrm{TPE}}(\lambda;\theta^{(s)}) - \gamma^{\mathrm{TPE}}(\lambda)^\top c_t(\theta^{(s)}) \right],
\qquad \theta^{(s)} \sim q_t,
\end{equation}
\vspace{-0.55cm}
\end{mybox}

with \vspace{-0.2cm}
\begin{equation}
\gamma^{\mathrm{TPE}}(\lambda) = \operatorname{Cov}(c_t,c_t)^{-1}
\operatorname{Cov}\!\left(c_t,h^{\mathrm{TPE}}(\lambda;\theta)\right).
\end{equation}

\subsection{\framework}
\label{sec:optimization_procedure}

We present the full \framework procedure in Alg.~\ref{algorithm}. Therein, we combine our orthogonalized acquisition estimator with ensemble surrogates and an outer log transformation.

\textbf{Ensemble surrogates.}
To mitigate structural surrogate misspecification, we consider an ensemble of surrogate models $m = 1,\dots,M$. Such ensembles can consist of Gaussian processes with different kernels, models with varying likelihood assumptions, various tree-based or neural surrogates. For each model $m$, we obtain the respective orthogonalized acquisition estimate $\widehat{\mathrm{EI}}_m^{\mathrm{orth}}(\lambda)$. We then ensemble the acquisition as a weighted average of weights $\pi_{m,t}$
\vspace{-0.2cm}
\begin{equation} \label{eq:ensemble}
\widehat{\mathrm{EI}}^{\mathrm{ens}}(\lambda) = \sum_{m=1}^M \pi_{m,t}
\widehat{\mathrm{EI}}_m^{\mathrm{orth}}(\lambda).
\end{equation}

\textbf{Exponentially weighted aggregation.}
We update ensemble weights through tempered exponentially weighted aggregation \cite{Freund.1997}, a widely used approach in model aggregation \cite{Fedus.2022,Jacobs.1991,Shazeer.2017}. A minimum weight floor \cite{Herbster.1998} prevents premature collapse onto a single surrogate
\begin{equation} \label{eq:weights}
\pi_{m,t} \propto \max(\delta,\, \pi_{m,t-1} \exp(\ell_{m,t}/\tau)),
\end{equation}
where $\ell_{m,t}$ is a predictive log score obtained from $m_t$ using $(\lambda_{t+1}, y_{t+1})$, $\tau > 0$ a temperature parameter, and $\delta > 0$ enforces a minimum weight. This results in a robust aggregation rule that balances exploitation of well-performing surrogates with persistent exploration across model classes. Note, that our overall methodology is independent and can be combined with other update strategies.  \vspace{0.2cm}
\begin{wrapfigure}[22]{r}{0.5\textwidth} 
\vspace{-.5cm}
\scriptsize
    \begin{algorithm}[H] 
    \caption{\footnotesize \framework}
    \label{algorithm}
    \KwIn{Initial data size $n_0$, budget $T$, number of surrogate models $M$, candidate set generator $\mathcal{C}_t$, numerical floor $\alpha > 0$, objective function $f$}
    Initialize $\mathcal{D}_0=\{(\lambda_i,y_i)\}_{i=1}^{n_0}$ using random  hyperparameter configurations drawn from a Sobol distribution\\
    \For{$t=0,1,\dots,T-1$}{
        \For{$m=1,\dots,M$}{
        Fit surrogate model $m$ on $\mathcal{D}_t$\;
        Obtain approximate posterior $q_{m,t}(\theta)$\;
        \vspace{-0.1cm}
    }
    Update ensemble weights $\pi_{m,t}$ for $m=1,\dots,M$\;
    \ForEach{$\lambda \in \mathcal{C}_t$}{
        \For{$m=1,\dots,M$}{
            \textbf{Orthogonalization:} 
            Compute orthogonalized marginal EI estimate $\widehat{\mathrm{EI}}_m^{\mathrm{orth}}(\lambda)$  via Eq.~\eqref{eq:ortho_EI}\;
            \vspace{-0.1cm}
        }
        \textbf{Ensemble strategy:}\vspace{-0.3cm} 
        \begin{align*} 
        \widehat{\mathrm{EI}}^{\mathrm{ens}}(\lambda)
        = \sum_{m=1}^M \pi_{m,t}\,\widehat{\mathrm{EI}}_m^{\mathrm{orth}}(\lambda)
        \end{align*} 
        \vspace{-0.4cm}
        }
        \textbf{Outer log transformation:} \vspace{-0.2cm}
        \begin{align*}
        \lambda_{t+1}
        = \arg\max_{\lambda \in \mathcal{C}_t}
        \log\!\big(\max(\widehat{\mathrm{EI}}^{\mathrm{ens}}(\lambda), \alpha)\big)        
        \end{align*}
        Evaluate $y_{t+1}=f(\lambda_{t+1})$;
        update $\mathcal{D}_{t+1}=\mathcal{D}_t \cup \{(\lambda_{t+1},y_{t+1})\}$\;
   \vspace{-0.1cm} 
}
\KwOut{Best observed configuration in $\mathcal{D}_T$}
\end{algorithm}
\normalsize
\end{wrapfigure}
\textbf{Optimization based on outer log stabilization.}
To improve numerical stability, we follow recent literature \cite{Ament.2023} and optimize the \emph{outer log transformed} acquisition
\begin{align}
    \widehat L^{\mathrm{ens}}(\lambda_t)=\log\!\big(\max(\widehat{\mathrm{EI}}^{\mathrm{ens}}(\lambda_t),\alpha)\big),
\end{align}
where $\alpha>0$ ensures that the logarithm is well defined. We select the next hyperparameter configuration $\lambda_{t+1} = \arg\max_{\lambda \in \mathcal{C}_t} \widehat L^{\mathrm{ens}}(\lambda)$, and evaluate it on the true objective to obtain $y_{t+1}$.

\vspace{-0.1cm}
\section{Empirical results}\label{sec:empirical_results}
\vspace{-0.1cm}

\textbf{Experimental setup.} $\bullet$\,\emph{Baselines:}
We evaluate \framework against various baselines from the literature. Specifically, we compare \framework against: \textbf{qLogEI} \citep{Ament.2023} (main comparison method), \textbf{UCB} \citep{wilson2017reparameterization}, \textbf{TuRBO} \citep{Eriksson.2019}, and \textbf{Sobol} sampling (e.g., \citep[]{Ament.2024}). Our experimental settings follow prior literature (e.g., \citep{Ament.2023,Ament.2024}): unless otherwise noted, all experiments use 16 replications, $S=512$, and we start all methods from the same $n_0=32$ Sobol points. For optimizing the acquisition functions, we use 512 raw samples and allow 8 restarts to prevent getting stuck in local minima \citep{torn1989global}. To isolate the effect of orthogonalization, we use $M=1$ throughout the experiments in the main paper. $\bullet$\,\emph{Datasets:} We consider four standard benchmarking functions with known global optima (to compute the regret) and varying dimensionality: \textbf{Hartmann6}, \textbf{Ackley8}, \textbf{Michalewicz10}, and \textbf{Levy16} (e.g., \citep[][]{Ament.2023,Bodin.2020,eriksson2021scalable,moss2023inducing}). Details of the benchmark functions and surrogate choices are provided in Supplement~\ref{appendix:implementation_details}. $\bullet$\,\emph{Metrics:} We compare methods using best-so-far regret and report 95\% confidence intervals (CIs) from standard normal approximation.

\textbf{Evaluation.} 
Our goal is to first demonstrate the theoretical properties of \framework by isolating the effect of orthogonalization of the acquisition estimation. For this, we show \rqtag{rq1}{1} the \textbf{variance reduction} and \rqtag{rq2}{2} \textbf{ranking stability} properties of \framework. Then, we evaluate the resulting \rqtag{rq3}{3} \textbf{efficiency gain} in terms of MC sampling budget. Finally, we demonstrate the effectiveness for real-world training and fine-tuning tasks where we show \rqtag{rq4}{4} the improved \textbf{downstream utility}. We present further experiments on misspecified surrogate families, weakly fitted hyperparameters or surrogates, CIFAR10, and ensemble experiments in Supplement~\ref{sec:appendix_results}.
 
\begin{wraptable}{r}{0.55\linewidth}
    \centering
    \scriptsize
    \vspace{-0.5cm}
    \caption{\textbf{Variance reduction} through orthogonalization (with Sobol probes; $\times 10^{-8}$).}
    \label{table:variance_reduction}
    \vspace{-0.2cm}
\begin{tabular}{p{0.1cm}p{1.25cm}p{1.25cm}p{1.25cm}p{1.22cm}}
\toprule
&   \multicolumn{2}{c}{Michalewicz10} & \multicolumn{2}{c}{Levy16}\\
\cmidrule(lr){2-3} \cmidrule(lr){4-5}
 $S$ & $\operatorname{Var}(\widehat{\mathrm{EI}}_m^{\mathrm{MC}})$ & $\operatorname{Var}(\widehat{\mathrm{EI}}_m^{\mathrm{ortho}})$ & $\operatorname{Var}(\widehat{\mathrm{EI}}_m^{\mathrm{MC}})$ & $\operatorname{Var}(\widehat{\mathrm{EI}}_m^{\mathrm{ortho}})$ \\
\midrule
8 & 9.50861 & 4.54151 & 105590.0 & 70935.8 \\
&& \tiny(-52.24\%) && \tiny(-32.82\%) \\
\midrule
32 & 1.66805 & 0.10681 & 20765.8 & 2266.7 \\
&& \tiny(-93.60\%) & & \tiny(-89.08\%) \\
\bottomrule
\end{tabular}
\vspace{-0.4cm}
\end{wraptable}
\textbf{\rqtag{rq1}{1} Variance reduction.}
We probe the acquisition function with random Sobol samples and compare the estimated variance of the raw and orthogonalized acquisition function. We use a GP surrogate. We report results for a Mat\'ern-$5/2$ kernel with ARD on Michalewicz10 and Levy16 in Table~\ref{table:variance_reduction}; results for an RBF kernel and a TPE surrogate are in Supplement~\ref{sec:appendix_var_stability_results}. $\Rightarrow$ \emph{Across all settings, \textbf{orthogonalization reduces the variance}, which confirms that \framework substantially stabilizes the acquisition estimate.}

\textbf{\rqtag{rq2}{2} Ranking stability.} 
We repeatedly evaluate the acquisition function on the same fixed Sobol probe set and compare the rankings across runs. We report: (i)~\emph{Probe variance}: mean variance of the acquisition value over the probe points across repeated acquisition evaluations; (ii)~\emph{Top1 agreement}: fraction of repeats that select the same top-ranked probe point; (iii)~the \emph{Flip rate} of adjacent high-ranked candidate pairs, and (iv)~\emph{Regret}. The flip rate mirrors the idea behind Proposition~\ref{prop:ranking}, in that it measures how often local pairwise orderings change due to noise in the acquisition-function estimation. The main results are in Table~\ref{table:ranking_stability}; further results are in Supplement~\ref{sec:appendix_var_stability_results}. $\Rightarrow$ \emph{Our variance reduction yields more stable rankings.} \hspace{8cm} \vspace{0.2cm}
\begin{wraptable}{r}{0.7\linewidth}
\vspace{-0.4cm}
    \centering
    \scriptsize
    \caption{\textbf{Ranking stability} (mean$\pm _{\mathrm{std}}$), via Sobol probes. Lower variance and flip rate indicate more stable acquisition estimates; higher top-1 agreement indicates more stable candidate rankings across probes.}
    \label{table:ranking_stability}
    \vspace{-0.2cm}
    \adjustbox{width=\linewidth}{
    \begin{tabular}{lrrrrrr}
\toprule
 & Method  & Probe &  Top1 & Flip & Regret $\downarrow$ \\
 && variance $\downarrow$ &  agreement $\uparrow$ & rate $\downarrow$ \\
\midrule
\multirow{2}{*}{Michalewicz10} & qLogEI & $168.603_{154.76}$ & $0.925_{0.16}$ & $0.148_{0.12}$ & $7.140_{0.01}$ \\
 & OrthoBO (ours) & $0.019_{0.03}$ & $0.988_{0.06}$ & $0.014_{0.03}$ & $6.961_{0.36}$ \\
\midrule
\multirow{2}{*}{Levy16} & qLogEI & $233.655_{287.67}$ & $0.917_{0.12}$ & $0.069_{0.09}$ &  $82.194_{6.05}$ \\
 & OrthoBO (ours) & $0.073_{0.21}$ & $0.988_{0.06}$ & $0.048_{0.04}$ &  $58.938_{0.00}$\\
\bottomrule
\end{tabular}}
\vspace{-0.3cm}
\end{wraptable}
\textbf{\rqtag{rq3}{3} MC efficiency.}
We use an isotropic RBF kernel (as a common smooth GP prior) combined with a moderate surrogate mismatch and reduced budges for the acquisition. We show the main results for varying $S$ on Ackley8 in Fig.~\ref{fig:mc_ablation}; further results are in Fig.~\ref{fig:mc_ablation_full} in Supplement~\ref{sec:appendix_results}. Reducing the MC budget makes the performance of acquisition-based methods more sensitive to estimation noise. Across all benchmarks, our \framework remains competitive and frequently achieves the lowest regret. $\Rightarrow$ \emph{Orthogonalization is especially effective in earlier iterations, where MC estimation noise has a larger effect on the acquisition values.}

\begin{figure}[h]
    \centering
    \includegraphics[width=1\textwidth]{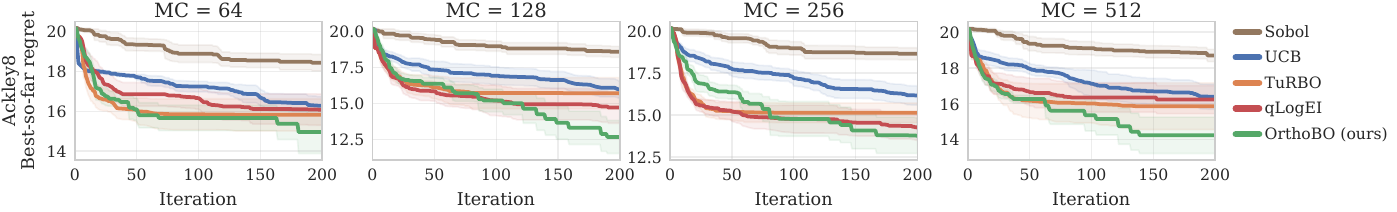}
    \vspace{-0.4cm}
    \caption{\textbf{MC efficiency.} Best-so-far regret for different MC budgets $S$ $\Rightarrow$ \emph{\framework achieves the lowest regrets across all budgets}.}
    \label{fig:mc_ablation}
\end{figure}
\textbf{\rqtag{rq4}{4} Downstream utility.}
We show how \framework improves hyperparameter optimization on two real-world tasks: (i)~training a neural network and (ii)~optimizing fine-tuning hyperparameters of a pre-trained vision transformer (ViT) \citep{dosovitskiy2020image} for manufacturing image classification.
\begin{wrapfigure}[12]{r}{0.45\linewidth}
    \vspace{-.2cm}
    \centering
    \includegraphics[width=\linewidth]{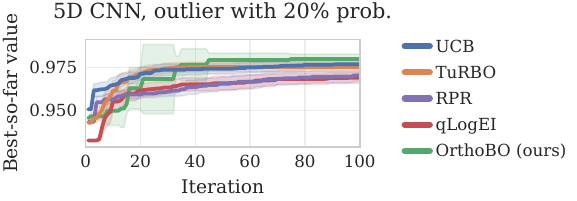}
    \vspace{-0.5cm}
    \caption{\textbf{Training (5D CNN).} Outliers (20\% prob.) are injected to make MNIST training more challenging. $\Rightarrow$ \emph{\framework achieves the \textbf{strongest final performance}.}}
    \label{fig:cnn_outlier}
\end{wrapfigure}

\vspace{-0.2cm}
\emph{(i) Training (neural network).} We employ the 5D CNN benchmark by \citet{Ament.2024} and include their RPR method as comparison. Details are in Supplement~\ref{sec:appendix_cnn_details}. To make the training more challenging, we introduce outliers with 20\% probability at different iterations, which can distort the surrogate fit, especially early in BO when only few observations are available. Results on MNIST \citep{lecun2010mnist} are in Fig.~\ref{fig:cnn_outlier}; CIFAR10 \citep{Krizhevsky09learningmultiple} results are in Fig.~\ref{fig:appendix_cifar10}. $\Rightarrow$~\emph{Here, \framework improves steadily over the optimization horizon and achieves the best value.}

\begin{wrapfigure}[10]{r}{0.45\linewidth}
    \centering
    \vspace{-0.6cm}
    \includegraphics[width=\linewidth]{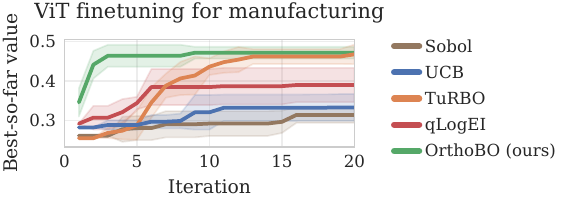}
    \vspace{-0.5cm}
    \caption{\textbf{Fine-tuning (ViT).} $\Rightarrow$ \emph{\framework \textbf{improves faster than the baselines} and achieves the highest best-so-far validation score.}}
    \label{fig:use_case}
\end{wrapfigure}
\emph{(ii) Fine-tuning (vision transformer).}
We optimize five fine-tuning hyperparameters on the industrial WM811K wafer-map dataset \citep{wafermap,wu2015wafer} for a vision transformer (ViT) based on the F1 score. Details are provided in Appendix~\ref{sec:appendix_case_study_details}. We present results in Fig.~\ref{fig:use_case}. \framework improves rapidly within the first few evaluations and maintains the highest best-so-far validation F1 score throughout most of the optimization horizon. $\Rightarrow$ \emph{Compared with qLogEI, \framework improves the best observed validation F1 score \textbf{by up to 20 percentage points}. This demonstrates that the \framework also improves performance in a practical, high-cost HPO setting beyond synthetic benchmarks.}

\vspace{-0.2cm}
\section{Discussion}\label{sec:discussion}
\vspace{-0.2cm}

$\bullet$ \emph{Conclusion.} We introduce \framework: a novel orthogonalized framework for reducing variance when estimating the acquisition function. As a result, our work identifies a failure mode of BO that was previously overlooked: BO decisions depend on the ranking of estimated acquisition values, so MC noise in these estimates can directly induce ranking errors and lead to suboptimal evaluations. \framework addresses this failure mode by reducing variance while preserving the underlying marginal EI target. Our results show that this variance reduction improves ranking stability and translates into better downstream BO performance. 
$\bullet$ \emph{Limitations.} First, \framework introduces additional computational overhead; however, the overhead is typically small compared with expensive evaluations of the objective function in HPO (see Supplement~\ref{sec:appendix_theory}). Second, we focus only on EI-based acquisition functions. Note that the same variance-reduction principle can be extended to other acquisition functions but may require acquisition-specific orthogonalization schemes (see Supplement~\ref{sec:appendix_general_extension}).
$\bullet$ \emph{Implications.} Our framework is widely applicable not only to HPO but essentially all BO problems, including in marketing and medicine.

\newpage
\bibliography{bibliography}

\newpage
\appendix

\section{Extended related work}
\label{sec:appendix_related_work}

Below, we discuss more distant related literature, completing the discussion in Section~\ref{sec:related_work}. We first cover literature regarding the robustness of the three components of BO in our work, and finally present related work on orthogonal machine learning.

\textbf{Marginalization over surrogate uncertainty.}
A standard source of instability in BO is uncertainty in the surrogate itself, especially in the hyperparameters such as kernel lengthscales. Prior work has emphasized that a fully Bayesian treatment of these hyperparameters can substantially improve robustness relative to empirical Bayes or plug-in optimization \cite{Snoek.2012}. Related work has also incorporated hyperparameter uncertainty directly into acquisition design \cite{Hernandez.2014, Hernandez.2015}. While these approaches account for surrogate uncertainty by integrating over posterior uncertainty in the model, they do not directly address the additional MC instability that arises when the resulting acquisition value must itself be estimated.

\textbf{Ensemble surrogates and model aggregation.}
A natural strategy for mitigating surrogate misspecification is to replace a single surrogate with an ensemble of surrogate models. Recent work has proposed GP ensembles and mixture-based surrogate models for BO, showing that combining multiple kernels or surrogate classes can improve robustness relative to committing to a single GP prior \cite{Lu.2023, Polyzos.2023}. These approaches address only structural misspecification at the model level and do not account for instability in the individual surrogates. 

\textbf{Acquisition functions and numerical stability.}
Acquisition functions are the core decision rule in BO, with classical choices including Probability of Improvement, Expected Improvement, and upper-confidence-based methods \cite{Jones.1998,Srinivas.2010}. Recent work has highlighted that vanilla EI and its variants can be numerically unstable, which can substantially degrade optimization performance \cite{Ament.2023}. This has motivated numerically stable reformulations based on the logarithmic scale, which are substantially easier to optimize in practice. 

\textbf{Orthogonal machine learning.}
Orthogonal machine learning is rooted in semi-parametric statistics \cite{Bickel.1993}, and is applied in many settings where the estimation target is defined based on unknown nuisance components \cite{Chernozhukov.2018, Foster.2023}. Very common examples are found in the missing data analysis \cite{Robins.1994} and causal inference literature \cite{Kennedy.2024, Nie.2021, Oprescu.2019, Robins.2000, vanderLaan.2006}. Recently, orthogonal learning has also been expanded to other fields such as private machine learning \cite{Schroder.2025}, survival analysis \cite{Frauen.2025}, LLM evaluation \cite{Frauen.2026}, and general ML optimization topics such as neural operators or generative models \cite{Hess.2026, Melnychuk.2025}.

\newpage
\section{Failure modes in Bayesian Optimization}\label{sec:appendix_bo_failure_modes}

\textbf{Failure modes in BO.} Several failure modes arise in a realistic HPO setting (see, e.g., \cite[e.g.][]{Ament.2023, Berkenkamp.2019, Bogunovic.2021, Neiswanger.2021}). Bayesian HPO performance depends on the accuracy and stability of acquisition estimates, which are computed from a surrogate posterior that is often \emph{misspecified and approximately inferred}, leading to instability in acquisition values. Specifically, BO suffers from the following main failure modes: 

\begin{enumerate}[wide, labelwidth=!, labelindent=-2pt]
\item[\textbf{(i)}] \emph{Structural surrogate misspecification} arises when the surrogate model class does not adequately represent the unknown objective. For example, the assumed kernel, likelihood, prior smoothness, or stationary structure may be incompatible with the true response surface. As a result, the posterior mean and posterior uncertainty estimates may be systematically biased. 
\item[\textbf{(ii)}] \emph{Numerical instability} can arise when acquisition values are very small, highly skewed, or nearly indistinguishable across candidates. Directly optimizing such quantities can lead to unstable gradients.
\item[\textbf{(iii)}] \emph{Posterior approximation and acquisition estimation noise.} Marginalizing over surrogate uncertainty mitigates plug-in overconfidence, but the marginal acquisition value must be approximated from a finite number of posterior samples. This introduces \emph{Monte Carlo variance} into the acquisition estimate, on top of any bias from the approximate posterior $q_{m,t}$ itself differing from the true posterior. Both effects can negatively affect the ranking of candidate points.
\end{enumerate}

\newpage
\section{Extension to constrained and parallel acquisitions}
\label{sec:appendix_constrained_parallel}

In Section~\ref{sec:method}, we develop an orthogonalized estimator for the standard Expected Improvement setting. The same construction extends to a broader class of acquisition functions, including constrained EI, parallel or batch EI, and their combinations. Below, we first provide a general extension to our standard EI setting in the main paper and then specialize it to constrained and parallel Bayesian optimization.

\subsection{General acquisition-level extension}
\label{sec:appendix_general_extension}

Let $z$ denote a generic acquisition input. Depending on the setting, we write $z$ to represent (i)~a single candidate point $\lambda$ in standard BO, (ii)~a single candidate point $\lambda$ in constrained BO, or (iii)~a batch $\Lambda = (\lambda_1,\dots,\lambda_q)$ in parallel BO.

For surrogate model $m$ with parameter $\theta_m$ and approximate posterior $q_{m,t}(\theta)$, let $a_m(z;\theta_m)$ denote a generic acquisition function computed under parameter value $\theta_m$. We define the marginal acquisition target
\begin{equation}
A_m(z) := \mathbb{E}_{\theta_m \sim q_{m,t}}\!\left[a_m(z;\theta_m)\right].
\label{eq:generic_marginal_acq}
\end{equation}

As in Section~\ref{sec:method}, let $g_m(\theta_m) := \nabla_{\theta_m}\log q_{m,t}(\theta_m)$ denote the score function. With $\theta_m^{(s)} \overset{\mathrm{i.i.d.}}{\sim} q_{m,t}$ the orthogonalized estimator is given by
\begin{equation}
\widehat{A}_m^{\mathrm{orth}}(z)
= \frac{1}{S}\sum_{s=1}^S
\left[ a_m(z;\theta_m^{(s)})
- \left(\operatorname{Cov}(g_m,g_m)^{-1}\operatorname{Cov}(g_m,a_m)\right)^\top g_m(\theta_m^{(s)})
\right].
\label{eq:generic_orth_acq}
\end{equation}

\begin{proposition}[Generic extension of the orthogonalized estimator]
\label{prop:generic_extension}
Fix an iteration $t$, a surrogate model $m$, and an acquisition input $z$. Assume that
\begin{equation}
\mathbb{E}_{q_{m,t}}\!\left[\|g_m(\theta_m)\|_2^2\right] < \infty,
\qquad
\mathbb{E}_{q_{m,t}}\!\left[a_m(\theta_m;z)^2\right] < \infty,
\end{equation}
and that $\operatorname{Cov}(g_m,g_m)$ is nonsingular. Then:
\begin{enumerate}
    \item $\widehat{A}_m^{\mathrm{orth}}(z)$ is unbiased for $A_m(z)$
    \item
    \begin{equation}
    \operatorname{Var}\!\left(\widehat{A}_m^{\mathrm{orth}}(z)\right)
    = \frac{1}{S} \left(\operatorname{Var}(h_m) -\operatorname{Cov}(h_m,g_m)^\top \operatorname{Cov}(g_m,g_m)^{-1} \operatorname{Cov}(g_m,h_m) \right),
    \end{equation}
    and therefore
    \begin{equation}
    \operatorname{Var}\!\left(\widehat{A}_m^{\mathrm{orth}}(z)\right)
    \le \operatorname{Var}\!\left(\widehat{A}_m^{\mathrm{MC}}(z)\right),
    \end{equation}
    where $\widehat{A}_m^{\mathrm{MC}}(z):=\frac{1}{S}\sum_{s=1}^S h_m(\theta_m^{(s)};z)$
    \item The residual $r_m(\theta_m;z)
    := a_m(\theta_m;z)-\left(\operatorname{Cov}(g_m,g_m)^{-1}\operatorname{Cov}(g_m,a_m)\right)^\top g_m(\theta_m)$ is orthogonal to the score directions.
\end{enumerate}
\end{proposition}
\begin{proof}
The proof is identical to that of the EI case in Supplement~\ref{sec:appendix_proofs} after replacing $\mathrm{EI}_m(\lambda;\theta_m)$ by the generic acquisition value $a_m(z;\theta_m)$.
\end{proof}

\paragraph{Remark.}
Proposition~\ref{prop:generic_extension} shows that the theoretical properties of our orthogonalized approach are acquisition-agnostic. We only require square integrability of the acquisition value under $q_{m,t}$ and the usual score-function regularity conditions.

\subsection{Constrained Expected Improvement}
\label{app:constrained_ei}

Consider constrained BO with one objective $f_1$ and constraint functions $f_2,\dots,f_M$, where, without loss of generality, the feasible set is defined by
\begin{equation}
f_i(\lambda)\le 0,
\qquad i=2,\dots,M.
\end{equation}
A standard constrained Expected Improvement (CEI) acquisition is
\begin{equation}
\mathrm{CEI}_m(\lambda;\theta_m)
:=
\mathbb{E}\!\left[
(f^\star - f_1(\lambda))_+
\prod_{i=2}^M \mathbf{1}\{f_i(\lambda)\le 0\}
\,\middle|\,
\mathcal{D}_t,\theta_m
\right],
\end{equation}
where $f^\star$ denotes the current best \emph{feasible} objective value (i.e., the minimum of $f_1$ over feasible observations).

The corresponding marginal target is $\mathrm{CEI}_m^{\mathrm{marg}}(\lambda)
:= \mathbb{E}_{\theta_m\sim q_{m,t}}
\!\left[\mathrm{CEI}_m(\lambda;\theta_m)\right]
\label{eq:cei_marg}$. Defining $h_m^{\mathrm{CEI}}(\theta_m;\lambda)
:= \mathrm{CEI}_m(\lambda;\theta_m)$,
the orthogonalized CEI estimator is given by
\begin{equation}
\widehat{\mathrm{CEI}}_m^{\mathrm{orth}}(\lambda)
=\frac{1}{S}\sum_{s=1}^S\left[\mathrm{CEI}_m(\lambda;\theta_m^{(s)})
-\left( \operatorname{Cov}(g_m,g_m)^{-1}
\operatorname{Cov}\!\left(g_m,h_m^{\mathrm{CEI}}\right) \right)^\top g_m(\theta_m^{(s)})\right].
\end{equation}

\begin{corollary}[Orthogonalized constrained EI]
Assume $\mathbb{E}_{q_{m,t}}
\!\left[
\left(\mathrm{CEI}_m(\lambda;\theta_m)\right)^2
\right]
<\infty$. Then Proposition~\ref{prop:generic_extension} applies with $a_m(z;\theta_m)=\mathrm{CEI}_m(\lambda;\theta_m)$.
\end{corollary}
\begin{proof}
The indicator product is bounded by $1$, so $0 \le \mathrm{CEI}_m(\lambda;\theta_m) \le \mathrm{EI}_{m,1}(\lambda;\theta_m) := \mathbb{E}\!\left[(f^\star - f_1(\lambda))_+ \mid \mathcal{D}_t,\theta_m\right]$. Square integrability of $\mathrm{EI}_{m,1}$ thus implies square integrability of $\mathrm{CEI}_m$, and the corollary follows from Proposition~\ref{prop:generic_extension}.
\end{proof}

\subsection{Parallel or batch Expected Improvement}
\label{app:parallel_qei}

In parallel BO, we select a batch $\Lambda=(\lambda_1,\dots,\lambda_q)$ of $q$ candidate points at each iteration. The standard batch expected improvement (qEI) is given by
\begin{equation}
\mathrm{qEI}_m(\Lambda;\theta_m) := \mathbb{E}\!\left[ \max_{j=1,\dots,q}
\big(f^\star - f(\lambda_j)\big)_+
\,\middle|\,
\mathcal{D}_t,\theta_m
\right],
\end{equation}
with corresponding marginal target is $\mathrm{qEI}_m^{\mathrm{marg}}(\Lambda) := \mathbb{E}_{\theta_m\sim q_{m,t}} \!\left[\mathrm{qEI}_m(\Lambda;\theta_m)\right]$.

We define $h_m^{\mathrm{qEI}}(\theta_m;\Lambda):=\mathrm{qEI}_m(\Lambda;\theta_m)$. Then the orthogonalized qEI estimator is given by
\begin{equation}
\widehat{\mathrm{qEI}}_m^{\mathrm{orth}}(\Lambda) = \frac{1}{S}\sum_{s=1}^S
\left[ \mathrm{qEI}_m(\Lambda;\theta_m^{(s)}) -
\left( \operatorname{Cov}(g_m,g_m)^{-1}
\operatorname{Cov}\!\left(g_m,h_m^{\mathrm{qEI}}\right) \right)^\top g_m(\theta_m^{(s)}) \right].
\label{eq:qei_orth}
\end{equation}

\begin{corollary}[Orthogonalized batch EI]
Assume $\mathbb{E}_{q_{m,t}} \!\left[\left(\mathrm{qEI}_m(\Lambda;\theta_m)\right)^2
\right]<\infty.$ Then Proposition~\ref{prop:generic_extension} applies with $a_m(z;\theta_m)=\mathrm{qEI}_m(\Lambda;\theta_m)$.
\end{corollary}
\begin{proof}
Observe that for every fixed batch $X=(x_1,\dots,x_q)$ and parameter value $\theta$, $0 \le \mathrm{qEI}_m(X;\theta)
\le
\sum_{j=1}^q \mathrm{EI}_m(x_j;\theta)$. Consequently, $\mathrm{qEI}_m(X;\theta)^2
\le
\left(\sum_{j=1}^q \mathrm{EI}_m(x_j;\theta)\right)^2$. In particular, if $\mathbb{E}_{\theta\sim q_{m,t}}
\left[
\left(\sum_{j=1}^q \mathrm{EI}_m(x_j;\theta)\right)^2
\right]
<\infty$, then as well $\mathbb{E}_{\theta\sim q_{m,t}}
\left[
\mathrm{qEI}_m(X;\theta)^2
\right]
<\infty.$ The corollary now directly follows from Proposition~\ref{prop:generic_extension}.
\end{proof}

Unlike standard EI, qEI does not admit a closed-form expression. Therefore, we follow common practice and approximate it by inner MC sampling. Let
\begin{equation}
\widehat{\mathrm{qEI}}_m(\Lambda;\theta_m)
=\frac{1}{N}\sum_{n=1}^N \max_{j=1,\dots,q}\big(f^\star - \xi^{(n)}(\lambda_j)\big)_+,\qquad
\xi^{(n)} \sim p(\cdot \mid \mathcal{D}_t,\theta_m),
\label{eq:qei_inner_mc}
\end{equation}
where the inner MC estimator is assumed to be unbiased for $\mathrm{qEI}_m(\Lambda;\theta_m)$.

The resulting orthogonalized MC estimator is then given by 
\begin{equation}
\widehat{\mathrm{qEI}}_{m,\mathrm{nested}}^{\mathrm{orth}}(\Lambda)
=\frac{1}{S}\sum_{s=1}^S
\left[\widehat{\mathrm{qEI}}_m(\Lambda;\theta_m^{(s)}) - (\operatorname{Cov}(g_m,g_m)^{-1}
\operatorname{Cov}\!(g_m,\widehat{{\mathrm{qEI}}}(\Lambda; \theta_m)))^\top g_m(\theta_m^{(s)})
\right].
\label{eq:qei_nested_orth}
\end{equation}

Importantly, under the assumption that the inner estimator \eqref{eq:qei_inner_mc} is unbiased for $\mathrm{qEI}_m(X;\theta_m)$ for every fixed $\theta_m$, and that all required second moments exist, then $\widehat{\mathrm{qEI}}_{m,\mathrm{nested}}^{\mathrm{orth}}(\Lambda)$ is unbiased for $\mathrm{qEI}_m(X;\theta_m)$ as well due to the law of iterated expectations.

\newpage
\section{Theoretical properties}
\label{sec:appendix_theory}

\subsection{Computational Complexity}
\label{sec:appendix_complexity}

\textbf{GP surrogate.}
For an exact GP surrogate, fitting requires $O(t^3)$ time due to kernel matrix inversion. Per candidate $\lambda$, computing $\mu_t(\lambda;\theta)$ and $\sigma_t^2(\lambda;\theta)$ for a fixed $\theta$ costs $O(t^2)$ in a naive implementation, or $O(t)$ if Cholesky factors are reused appropriately. For $S$ posterior samples, marginal EI estimation costs $O(S t^2)$ or $O(S t)$, depending on the implementation.

The orthogonalization step introduces two types of cost. First, once per BO iteration, one computes the score covariance matrix $\widehat{\Sigma}_g := \widehat{\operatorname{Cov}}(g_t,g_t) \in \mathbb{R}^{d_\theta \times d_\theta}$, which depends only on the surrogate posterior samples and not on the candidate $\lambda$. This costs $O(S d_\theta^2)$, followed by $O(d_\theta^3)$ to invert or factorize $\widehat{\Sigma}_g$. Second, for each candidate $\lambda$, one computes the cross-covariance vector
$\widehat{c}_{gh}(\lambda) := \widehat{\operatorname{Cov}}(g_t,h(\cdot;\lambda))$, which costs $O(S d_\theta)$, followed by $O(d_\theta^2)$ for a matrix--vector solve or multiplication with the precomputed factorization of $\widehat{\Sigma}_g^{-1}$. Hence, the additional per-candidate orthogonalization overhead is $O(S d_\theta + d_\theta^2)$,
and the total per-candidate cost is $O(S t^2 + S d_\theta + d_\theta^2)$,
plus a one-time per-iteration preprocessing cost of $O(S d_\theta^2 + d_\theta^3)$.

We compare the average per-step runtimes of \framework with qLogEI in Table~\ref{tab:runtime_comparison}. We observe that the runtimes vary with the dimensionality of the problem, but the overhead over MC-based qLogEI is often around \SI{2}{\second}. For our implementation, we used the \texttt{BoTorch} framework, which has made MC-sampling highly efficient. We thus expect further runtime reductions from a more improved orthogonalization implementation.

\begin{table}[h]
    \centering
    \caption{Average per-step runtimes (in seconds) for \framework and qLogEI.}
    \label{tab:runtime_comparison}
    \begin{tabular}{lrrrr}
\toprule
 & Hartmann6 & Ackley8 & Michalewicz10 & Levy16 \\
\midrule
qLogEI & 0.696 & 0.156 & 0.325 & 0.582 \\
\framework & 1.607 & 0.492 & 2.362 & 2.563 \\
\bottomrule
\end{tabular}

\end{table}

\textbf{Tree-based / TPE-style surrogate.}
Tree-based or density-based surrogates are typically cheaper to fit. A single fit often requires $O(t \log t)$ or comparable cost, depending on the underlying model. Per-candidate acquisition evaluation is usually $O(1)$ or logarithmic in $t$. As in the GP case, the covariance matrix of the control variate, $\widehat{\Sigma}_c := \widehat{\operatorname{Cov}}(c_t,c_t) \in \mathbb{R}^{d_c \times d_c}$,
depends only on the posterior samples and can therefore be computed once per BO iteration at cost $O(S d_c^2)$, followed by $O(d_c^3)$ for inversion or factorization. For each candidate, one then computes the cross-covariance vector $\widehat{c}_{ch}(\lambda) := \widehat{\operatorname{Cov}}(c_t,h(\cdot;\lambda))$,
at cost $O(S d_c)$, followed by $O(d_c^2)$ for the matrix--vector solve. Hence, the additional per-candidate orthogonalization overhead is
$O(S d_c + d_c^2)$, with one-time per-iteration preprocessing cost
$O(S d_c^2 + d_c^3)$.

We report average per-step runtimes for orthogonalized TPE versus MC-based TPE in Table~\ref{tab:tpe_runtimes}. Runtime differences are generally small.

\begin{table}[h]
    \centering
        \caption{Average per-step runtimes (in seconds) for \framework and MC-based TPE.}
    \label{tab:tpe_runtimes}
\begin{tabular}{lrrrr}
\toprule
 & Hartmann6 & Ackley8 & Michalewicz10 & Levy16 \\
\midrule
MC-TPE & 0.224 & 0.347 & 0.356 & 0.754 \\
\framework & 0.390 & 0.736 & 2.533 & 0.360 \\
\bottomrule
\end{tabular}

\end{table}

\textbf{Ensemble overhead.}
For an ensemble of $M$ surrogate models, the total cost scales linearly as $O\!\left(\sum_{m=1}^M \text{cost}_m\right).$ In practice, the main computational advantage of debiasing through orthogonalization is not that it makes each acquisition evaluation cheaper, but that it can achieve comparable acquisition stability with a smaller MC budget $S$.

\subsection{Variance reduction, robustness and ranking stability for GP and TPE}

We provide versions of Theorem~\ref{thm:variance_reduction},  and Proposition~\ref{prop:ranking} for the two surrogate instantiations GP and TPE. The proofs follow directly from the proofs of the original theorem and proposition in Supplement~\ref{sec:appendix_proofs}. 

\begin{theorem}[Variance reduction for the GP surrogate]
\label{thm:variance_gp}
Fix an iteration $t$ and a candidate $\lambda \in \Lambda$. Assume $\mathbb{E}_{q_t}\|g_t(\theta)\|_2^2 < \infty$, $\mathbb{E}_{q_t}[(\mathrm{EI}^{\mathrm{GP}}(\lambda;\theta))^2] < \infty$, and that $\operatorname{Cov}(g_t,g_t)$ is nonsingular. Define
\begin{align}
    \mathrm{EI}^{\mathrm{GP,orth}}(\lambda;\theta) := \mathrm{EI}^{\mathrm{GP}}(\lambda;\theta) - \gamma^{\mathrm{GP}}(\lambda)^\top g_t(\theta).
\end{align}
Then:

(i) \textbf{Target preservation:} $\mathbb{E}_{q_t}[\mathrm{EI}^{\mathrm{GP,orth}}(\lambda;\theta)] = \mathrm{EI}^{\mathrm{GP,marg}}(\lambda)$.

(ii) \textbf{Variance reduction:}
    \begin{align}
    \operatorname{Var}\!\big(\mathrm{EI}^{\mathrm{GP,orth}}(\lambda;\theta)\big)
    = \operatorname{Var}\!\big(\mathrm{EI}^{\mathrm{GP}}(\lambda;\theta)\big)
    - \operatorname{Cov}(\mathrm{EI}^{\mathrm{GP}},g_t)^\top
    \operatorname{Cov}(g_t,g_t)^{-1}
    \operatorname{Cov}(g_t,\mathrm{EI}^{\mathrm{GP}}),
    \end{align}
and therefore $\operatorname{Var}(\mathrm{EI}^{\mathrm{GP,orth}}(\lambda;\theta)) \le \operatorname{Var}(\mathrm{EI}^{\mathrm{GP}}(\lambda;\theta))$.
By the i.i.d.\ MC variance identity, the corresponding estimator satisfies $\operatorname{Var}(\widehat{\mathrm{EI}}^{\mathrm{GP,orth}}(\lambda)) \le \operatorname{Var}(\widehat{\mathrm{EI}}^{\mathrm{GP,MC}}(\lambda))$ for any MC budget $S$.
\end{theorem}

\begin{proposition}[Ranking stability for the GP surrogate]
\label{prop:ranking_gp}
Let $\lambda,\lambda' \in \Lambda$ with $\Delta^{\mathrm{GP}}(\lambda,\lambda') := \mathrm{EI}^{\mathrm{GP,marg}}(\lambda) - \mathrm{EI}^{\mathrm{GP,marg}}(\lambda') > 0$. Let $\widehat\Delta^{\mathrm{GP,orth}}$ denote the MC estimator obtained by applying the orthogonalization construction in Theorem~\ref{thm:variance_gp} to the difference functional $\mathrm{EI}^{\mathrm{GP}}(\lambda;\theta) - \mathrm{EI}^{\mathrm{GP}}(\lambda';\theta)$. Then
\begin{align}
\mathbb{P}\!\left(\widehat\Delta^{\mathrm{GP,orth}} \le 0\right)
\;\le\; \frac{\operatorname{Var}(\widehat\Delta^{\mathrm{GP,orth}})}{(\Delta^{\mathrm{GP}}(\lambda,\lambda'))^2}
\;\le\; \frac{\operatorname{Var}(\widehat\Delta^{\mathrm{GP,MC}})}{(\Delta^{\mathrm{GP}}(\lambda,\lambda'))^2}.
\end{align}
\end{proposition}

\begin{theorem}[Variance reduction for the tree-based / TPE-style surrogate]
\label{thm:variance_tpe}
Fix an iteration $t$ and a candidate $\lambda \in \Lambda$. Assume $\mathbb{E}_{q_t}[\|c_t(\theta)\|_2^2] < \infty$, $\mathbb{E}_{q_t}[(h^{\mathrm{TPE}}(\lambda;\theta))^2] < \infty$, $\mathbb{E}_{q_t}[c_t(\theta)] = 0$, and that $\operatorname{Cov}(c_t,c_t)$ is nonsingular. Define
\begin{align}
    h^{\mathrm{TPE,orth}}(\lambda;\theta) := h^{\mathrm{TPE}}(\lambda;\theta) - \gamma^{\mathrm{TPE}}(\lambda)^\top c_t(\theta).
\end{align}
Then:

(i) \textbf{Target preservation:} $\mathbb{E}_{q_t}[h^{\mathrm{TPE,orth}}(\lambda;\theta)] = \mathrm{EI}^{\mathrm{TPE,marg}}(\lambda)$.

(ii) \textbf{Variance reduction:}
    \begin{align}
    \operatorname{Var}\!\big(h^{\mathrm{TPE,orth}}(\lambda;\theta)\big)
    = \operatorname{Var}\!\big(h^{\mathrm{TPE}}(\lambda;\theta)\big)
    - \operatorname{Cov}(h^{\mathrm{TPE}},c_t)^\top
    \operatorname{Cov}(c_t,c_t)^{-1}
    \operatorname{Cov}(c_t,h^{\mathrm{TPE}}),
    \end{align}
and therefore $\operatorname{Var}(h^{\mathrm{TPE,orth}}(\lambda;\theta)) \le \operatorname{Var}(h^{\mathrm{TPE}}(\lambda;\theta))$.
By the i.i.d.\ MC variance identity, the corresponding estimator satisfies $\operatorname{Var}(\widehat{\mathrm{EI}}^{\mathrm{TPE,orth}}(\lambda)) \le \operatorname{Var}(\widehat{\mathrm{EI}}^{\mathrm{TPE,MC}}(\lambda))$ for any MC budget $S$.
\end{theorem}

\begin{proposition}[Ranking stability for the tree-based / TPE-style surrogate]
\label{prop:ranking_tpe}
Let $\lambda,\lambda' \in \Lambda$ with $\Delta^{\mathrm{TPE}}(\lambda,\lambda') := \mathrm{EI}^{\mathrm{TPE,marg}}(\lambda) - \mathrm{EI}^{\mathrm{TPE,marg}}(\lambda') > 0$. Let $\widehat\Delta^{\mathrm{TPE,orth}}$ denote the MC estimator obtained by applying the orthogonalization construction in Theorem~\ref{thm:variance_tpe} to the difference functional $h^{\mathrm{TPE}}(\lambda;\theta) - h^{\mathrm{TPE}}(\lambda';\theta)$. Then
\begin{align}
\mathbb{P}\!\left(\widehat\Delta^{\mathrm{TPE,orth}} \le 0\right)
\;\le\; \frac{\operatorname{Var}(\widehat\Delta^{\mathrm{TPE,orth}})}{(\Delta^{\mathrm{TPE}}(\lambda,\lambda'))^2}
\;\le\; \frac{\operatorname{Var}(\widehat\Delta^{\mathrm{TPE,MC}})}{(\Delta^{\mathrm{TPE}}(\lambda,\lambda'))^2}.
\end{align}
\end{proposition}

\newpage
\section{Proofs of the main theorems and propositions}
\label{sec:appendix_proofs}

\subsection{Proof of Theorem~\ref{thm:variance_reduction}}\label{sec:proof_variance_reduction}

\textbf{Theorem~\ref{thm:variance_reduction}}
Assume $\mathbb{E}_{q_{m,t}}[\mathrm{EI}_m(\lambda;\theta)^2] < \infty$ and $\mathbb{E}_{q_{m,t}}[\|g(\theta)\|_2^2] < \infty$, that $\Sigma_g$ is nonsingular, and that the standard score-function regularity conditions hold for $q_{m,t}$. Then:

\textbf{(i) Preserved target.}
$\mathrm{EI}_m^{\mathrm{orth, marg}}(\lambda) = \mathrm{EI}_m^{\mathrm{marg}}(\lambda)$.

\textbf{(ii) Variance reduction.} Under $q_{m,t}$,
\begin{equation}
    \operatorname{Var}\!\big(\mathrm{EI}_m^{\mathrm{orth}}(\lambda; \theta)\big)
    = \operatorname{Var}\!\big(\mathrm{EI}_m(\lambda; \theta)\big)
    - \operatorname{Cov}(g, \mathrm{EI}_m)^\top \Sigma_g^{-1} \operatorname{Cov}(g, \mathrm{EI}_m) \leq \operatorname{Var}\!\big(\mathrm{EI}_m(\lambda; \theta)\big).
\end{equation}

\begin{proof}
Let $h(\theta) := \mathrm{EI}_m(\lambda;\theta)$ and $g(\theta) := \nabla_\theta \log q_{m,t}(\theta)$, and write $\gamma := \gamma_m(\lambda) = \Sigma_g^{-1}\operatorname{Cov}(g, h)$. Then
\begin{align}
\mathrm{EI}_m^{\mathrm{orth}}(\lambda; \theta) = h(\theta) - \gamma^\top g(\theta).
\end{align}

\paragraph{(i) Preserved target.}
Under the assumed regularity conditions, the score function satisfies
\begin{align}
\mathbb{E}_{q_{m,t}}[g(\theta)]
= \int \nabla_\theta \log q_{m,t}(\theta)\, q_{m,t}(\theta)\, d\theta
= \int \nabla_\theta q_{m,t}(\theta)\, d\theta
= 0.
\end{align}
Hence
\begin{align}
\mathrm{EI}_m^{\mathrm{orth, marg}}(\lambda)
= \mathbb{E}_{q_{m,t}}[h(\theta)] - \gamma^\top \mathbb{E}_{q_{m,t}}[g(\theta)]
= \mathbb{E}_{q_{m,t}}[h(\theta)]
= \mathrm{EI}_m^{\mathrm{marg}}(\lambda).
\end{align}

\paragraph{(ii) Variance reduction.}
For any $a \in \mathbb{R}^{d_\theta}$,
\begin{align}
\operatorname{Var}_{q_{m,t}}(h - a^\top g)
= \operatorname{Var}(h) - 2 a^\top \operatorname{Cov}(g, h) + a^\top \Sigma_g\, a.
\end{align}
Since $\Sigma_g$ is positive definite by assumption, this quadratic in $a$ is uniquely minimized at
\begin{align}
a^\star = \Sigma_g^{-1} \operatorname{Cov}(g, h) = \gamma.
\end{align}
Substituting $a = \gamma$ gives
\begin{align}
\operatorname{Var}\!\big(\mathrm{EI}_m^{\mathrm{orth}}(\lambda; \theta)\big)
= \operatorname{Var}(h) - \operatorname{Cov}(g, h)^\top \Sigma_g^{-1} \operatorname{Cov}(g, h).
\end{align}
Because $\Sigma_g$ is positive definite, the quadratic form $\operatorname{Cov}(g, h)^\top \Sigma_g^{-1} \operatorname{Cov}(g, h) \ge 0$, and therefore
\begin{align}
\operatorname{Var}\!\big(\mathrm{EI}_m^{\mathrm{orth}}(\lambda; \theta)\big)
\;\leq\; \operatorname{Var}\!\big(\mathrm{EI}_m(\lambda; \theta)\big).
\end{align}
\end{proof}

\subsection{Proof of Corollary~\ref{cor:robustness}}

\textbf{Corollary~\ref{cor:robustness}}
Assume the conditions of Theorem~\ref{thm:variance_reduction}. For any $b \in \mathbb{R}^{d_\theta}$, consider the tilted family $q_{m,t}^{(\varepsilon)}(\theta) \propto q_{m,t}(\theta)\exp\!\big(\varepsilon\, b^\top g(\theta)\big)$ for small $|\varepsilon|$. Holding $\gamma_m(\lambda)$ fixed at its $\varepsilon=0$ value, it holds that
\vspace{-0.2cm}
\begin{equation}
    \left.\frac{\diff}{\diff\varepsilon}\, \mathbb{E}_{q_{m,t}^{(\varepsilon)}}\!\left[\mathrm{EI}_m^{\mathrm{orth}}(\lambda; \theta)\right]
    \right|_{\varepsilon=0}
    \;=\; 0.
\end{equation}

\begin{proof}
Throughout, we hold $\gamma$ fixed at its $\varepsilon = 0$ value. Consider the exponentially tilted family
\begin{align}
q_{m,t}^{(\varepsilon)}(\theta)
= \frac{q_{m,t}(\theta)\exp(\varepsilon\, b^\top g(\theta))}{Z(\varepsilon)},
\qquad Z(\varepsilon) = \int q_{m,t}(\theta)\exp(\varepsilon\, b^\top g(\theta))\, d\theta.
\end{align}
For any integrable test function $\varphi$, differentiation under the integral sign (justified by the square-integrability assumptions and dominated convergence in a neighborhood of $\varepsilon = 0$) yields
\begin{align}
\frac{d}{d\varepsilon}\,\mathbb{E}_{q_{m,t}^{(\varepsilon)}}[\varphi(\theta)]
\bigg|_{\varepsilon=0}
= \operatorname{Cov}_{q_{m,t}}\!\big(\varphi(\theta), b^\top g(\theta)\big).
\end{align}
Apply with $\varphi(\theta) = \mathrm{EI}_m^{\mathrm{orth}}(\lambda; \theta) = h(\theta) - \gamma^\top g(\theta)$:
\begin{align}
\operatorname{Cov}\!\big(g,\, h - \gamma^\top g\big)
= \operatorname{Cov}(g, h) - \Sigma_g\, \gamma
= \operatorname{Cov}(g, h) - \Sigma_g \Sigma_g^{-1} \operatorname{Cov}(g, h)
= 0.
\end{align}
Hence
\begin{align}
\left.\frac{d}{d\varepsilon}\,
\mathbb{E}_{q_{m,t}^{(\varepsilon)}}\!\left[\mathrm{EI}_m^{\mathrm{orth}}(\lambda; \theta)\right]
\right|_{\varepsilon=0}
= b^\top \operatorname{Cov}\!\big(g,\, h - \gamma^\top g\big) = 0
\end{align}
for every $b \in \mathbb{R}^{d_\theta}$. The orthogonalized acquisition is therefore first-order insensitive to score-tilt perturbations of $q_{m,t}$.
\end{proof}

\subsection{Proof of Proposition~\ref{prop:ranking}}

\textbf{Proposition~\ref{prop:ranking}}
Let $\lambda, \lambda' \in \Lambda$ such that $\Delta(\lambda,\lambda') := \mathrm{EI}_m^{\mathrm{marg}}(\lambda) - \mathrm{EI}_m^{\mathrm{marg}}(\lambda') > 0$. Let $\widehat{\Delta}_{\mathrm{MC}}(\lambda,\lambda')$ and $\widehat{\Delta}_{\mathrm{orth}}(\lambda,\lambda')$ denote the corresponding Monte Carlo difference estimators. 
Then both estimators are unbiased for $\Delta(\lambda,\lambda')$, and 
\begin{equation}
    \mathbb{P}\!\left(\widehat{\Delta}_{\mathrm{orth}}(\lambda,\lambda') \le 0\right)
    \;\leq\; \frac{\operatorname{Var}\!\left(\widehat{\Delta}_{\mathrm{orth}}(\lambda,\lambda')\right)}{\operatorname{Var}\!\left(\widehat{\Delta}_{\mathrm{orth}}(\lambda,\lambda')\right) + \Delta(\lambda,\lambda')^2}
    \;\leq\; \frac{\operatorname{Var}\!\left(\widehat{\Delta}_{\mathrm{MC}}(\lambda,\lambda')\right)}{\operatorname{Var}\!\left(\widehat{\Delta}_{\mathrm{MC}}(\lambda,\lambda')\right) + \Delta(\lambda,\lambda')^2}.
\end{equation}

\begin{proof}
By unbiasedness from Theorem~\ref{thm:variance_reduction},
\begin{equation}
\mathbb{E}\!\left[\widehat{\Delta}_{\mathrm{orth}}(\lambda,\lambda')\right]
= \Delta(\lambda,\lambda').
\end{equation}
Hence
\begin{equation}
\mathbb{P}\big(\widehat{\Delta}_{\mathrm{orth}}(\lambda,\lambda') \le 0\big)
= \mathbb{P}\!\left(\widehat{\Delta}_{\mathrm{orth}}(\lambda,\lambda') - \Delta(\lambda,\lambda') \le -\Delta(\lambda,\lambda') \right).
\end{equation}
Applying Cantelli's inequality to the centered random variable
\begin{equation}
Z:=\widehat{\Delta}_{\mathrm{orth}}(\lambda,\lambda')-\Delta(\lambda,\lambda')
\end{equation}
gives
\begin{equation}
\mathbb{P}(Z \le -a) \le \frac{\operatorname{Var}(Z)}{\operatorname{Var (Z)+a^2}}, \qquad a>0.
\end{equation}
With $a=\Delta(\lambda,\lambda')$, this yields
\begin{equation}
\mathbb{P}\big(\widehat{\Delta}_{\mathrm{orth}}(\lambda,\lambda') \le 0\big)
\le \frac{\operatorname{Var}\!\left(\widehat{\Delta}_{\mathrm{orth}}(\lambda,\lambda')\right)}{
\operatorname{Var}\!\left(\widehat{\Delta}_{\mathrm{orth}}(\lambda,\lambda')\right) + \Delta(\lambda,\lambda')^2}.
\end{equation}
The second inequality follows from Theorem~\ref{thm:variance_reduction}, which gives
\begin{equation}
\operatorname{Var}\!\left(\widehat{\Delta}_{\mathrm{orth}}(\lambda,\lambda')\right) \le \operatorname{Var}\!\left(\widehat{\Delta}_{\mathrm{MC}}(\lambda,\lambda')\right).
\end{equation}
Since the map $v \mapsto v/(v+\Delta^2)$ is increasing for $v \ge 0$, the result follows.
\end{proof}

\newpage
\section{Implementation details}\label{appendix:implementation_details}

\textbf{General implementation details} All experiments were conducted on a proprietary compute cluster. Our experiments utilized 4 CPU cores, 16 GB of main memory, and a A100 MIG slice with 7GB of VRAM. We used \texttt{BoTorch} for all code implementations. In total, we use $\sim$ 1000 GPU hours including initial experimentation and bug fixing.

\textbf{License details} We used BoTorch, which is MIT licensed. We implemented all methods and benchmarks in this framework. The WM811K dataset \citep{wafermap,wu2015wafer} is available at \url{http://mirlab.org/dataSet/public/} and allows \emph{usage} under the following conditions (copied verbatim from the license document): 
\begin{verbatim}
    2.	Redistribution and use in any form must be 
    accompanied by the following two citations: 

	[1] Ming-Ju Wu, Jyh-Shing Roger Jang, and Jui-Long Chen, 
    "Wafer Map Failure Pattern Recognition and
    Similarity Ranking for Large-Scale Data Sets," 
    in IEEE Transactions on Semiconductor Manufacturing, 
    vol. 28, no. 1, pp. 1-12, Feb. 2015, doi: 10.1109/TSM.2014.2364237.

    [2] MIR-WM811K: Dataset for wafer map failure pattern recognition,
    2015 http://mirlab.org/dataset/public/
\end{verbatim}
We included these citations at the appropriate places. Further, we received written permission by the authors to use the dataset for our research.

\subsection{TPE implementation}
\paragraph{TPE implementation.}
For the TPE-based experiments, we approximate the parameter distribution $q_t(\theta)$ by a nonparametric bootstrap over the current BO history. Concretely, for each sample $s=1,\dots,S$, we resample the observed data with replacement, fit a TPE surrogate $\theta^{(s)}$, and evaluate the corresponding TPE acquisition
\[
h_{\mathrm{TPE}}(\lambda;\theta^{(s)}) \propto \frac{\ell_{\theta^{(s)}}(\lambda)}{g_{\theta^{(s)}}(\lambda)} .
\]
We use the empirical $\gamma$-quantile split with $\gamma=0.2$ to define the good and bad sets. 

As control variate, we use centered bootstrap summary statistics. For each bootstrap fit $\theta^{(s)}$, we form
\[
\phi(\theta^{(s)})=
\big[
\log b^{(s)}_{\mathrm{good},1},\dots,\log b^{(s)}_{\mathrm{good},d},
\log b^{(s)}_{\mathrm{bad},1},\dots,\log b^{(s)}_{\mathrm{bad},d},
y^{\star,(s)}
\big]^\top,
\]
where $b^{(s)}_{\mathrm{good},j}$ and $b^{(s)}_{\mathrm{bad},j}$ are the fitted KDE bandwidths in dimension $j$, and $y^{\star,(s)}$ is the bootstrap split threshold. The implemented control variate is then
\[
c_t(\theta^{(s)})=
\phi(\theta^{(s)})-\frac{1}{S}\sum_{r=1}^S \phi(\theta^{(r)}),
\]
which is zero-mean by construction. The orthogonalized TPE estimator is
\[
\widehat{\mathrm{EI}}^{\,c}_{\mathrm{TPE,orth}}(\lambda)
=
\frac{1}{S}\sum_{s=1}^S
\left(
h_{\mathrm{TPE}}(\lambda;\theta^{(s)})
-
\gamma_{\mathrm{TPE}}(\lambda)^\top c_t(\theta^{(s)})
\right).
\]

\subsection{Misspecified kernel family experiments}\label{appendix:misspecified_kernel_families_details}

\begin{table*}[htbp]
    \centering
    \small
    \setlength{\tabcolsep}{4pt}
    \renewcommand{\arraystretch}{1.15}
    \begin{tabularx}{\textwidth}{
        >{\raggedright\arraybackslash}p{2cm}
        >{\raggedright\arraybackslash}p{3.cm}
        >{\raggedright\arraybackslash}p{2cm}
        >{\raggedright\arraybackslash}p{2.75cm}
        >{\raggedright\arraybackslash}X
    }
        \toprule
        \textbf{Function}
        & \textbf{Qualitative structure}
        & \multicolumn{2}{c}{\textbf{Kernel family}}
        & \textbf{Reason for misspecification} \\
        \cmidrule(lr){3-4}
        &
        & \textbf{Mildly misspecified}
        & \textbf{Strongly misspecified}
        & \\
        \midrule

        Hartmann6
        & Smooth, anisotropic, interaction-heavy, multimodal
        & Isotropic RBF or isotropic Mat\'ern
        & Linear
        & Isotropic kernel variants miss dimension-specific lengthscales. Linear kernels cannot capture the nonlinear multimodal structure. \\

        \midrule
        Ackley8 
        & Highly multimodal, oscillatory, with many local minima
        & Smooth RBF
        & Linear
        & Very smooth kernels tend to smear out Ackley's rugged local structure. Linear kernels are too simple for the landscape. \\

        \midrule
        Michalewicz10
        & Sharp valleys, strong multimodality, narrow optima
        & Isotropic RBF or Mat\'ern
        & Linear
        & Isotropic kernels blur the narrow valleys and local optima. Linear kernels are strongly misspecified. \\

        \midrule
        Levy16
        & Rugged, oscillatory, with nonlinear interactions
        & Overly smooth isotropic RBF
        & Linear
        & Isotropic smooth kernels underfit sharper local variation and ignore anisotropy. Linear kernels cannot represent the nonlinear oscillatory structure. \\

        \bottomrule
    \end{tabularx}
    \caption{Synthetic benchmark functions, their qualitative structure, and examples of mildly and strongly misspecified kernel families.}
    \label{tab:function_wrong_kernels}
\end{table*}

\subsection{CNN outlier experiment}\label{sec:appendix_cnn_details}
We include the CNN outlier experiment used by \citet{Ament.2024}. Specifically, the goal is to optimize the test-set accuracy of a CNN trained on MNIST images. To model corrupted evaluations, we prematurely stop training with probability \SI{20}{\percent} after observing between 100 and 1000 training samples. We optimize the same parameters as \citet{Ament.2024}. For comparison, we also include their RPR method. To not implicitly mitigate the corrupted evaluations, we only use $n_0=2$ initial points. In initial experiments, we found that using larger $n_0$ made the problem \emph{too easy for all methods}; all reached $\sim$\SI{100}{\percent} test accuracy after few iterations.

\subsection{Case study experiment, additional details}\label{sec:appendix_case_study_details}

\textbf{Dataset.} For our real-world case study, we use the large-scale industrial WM811K dataset \citep{wafermap,wu2015wafer}. It contains \num{811457} real-world wafer map images and annotations of common failure types. The images have pixel values $0,1,2$, where $0$ is background, $1$ are OK regions, and $2$ are faulty regions. Depending on the clustering of the failures, one of eight class labels is assigned. Exemplary failure patterns are shown in Figure~\ref{fig:appendix_wm811k_examples}. Images without any failures are labelled with an additional class label. We resize all wafer maps to a common ($48 \times48$) resolution, rescale the data to $(0,1)$ by dividing by $/2$, and then normalize using ImageNet statistics.

\begin{figure}
    \centering
    \includegraphics[width=0.5\linewidth]{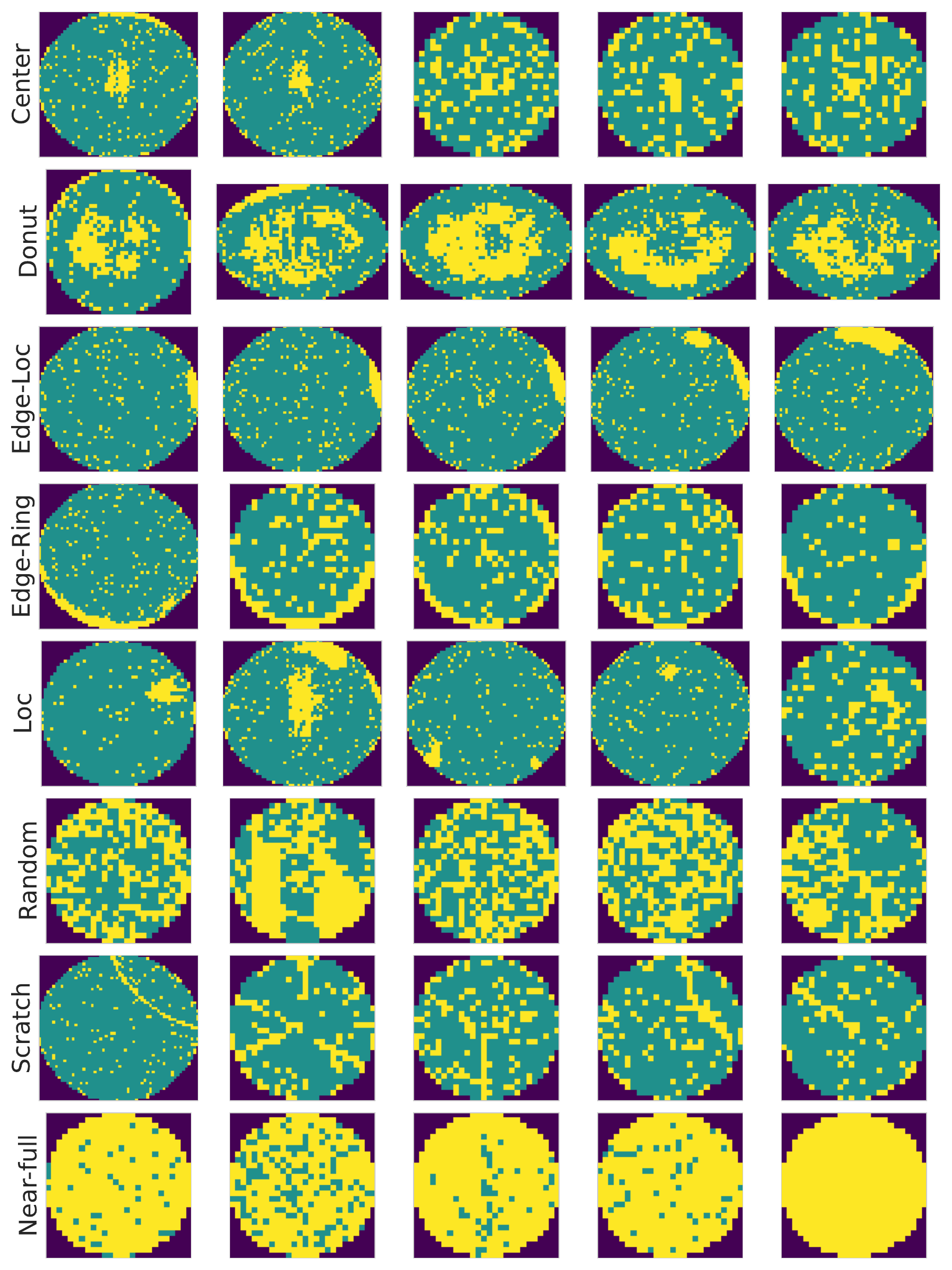}
    \caption{Examples of wafer-map manufacturing failures.}
    \label{fig:appendix_wm811k_examples}
\end{figure}

The dataset is strongly imbalanced, and we randomly split the data into a \SI{80}{\percent} train and \SI{20}{\percent} test set, stratifying on the labels to keep class distributions equal across the two splits. During fine-tuning, we use an (inverse) class-weighted cross-entropy loss.

\textbf{Model.} As a model, we use an ImageNet \citep{deng2009imagenet} pretrained ViT B16 \citep{dosovitskiy2020image} model. We freeze the backbone model and train only the penultimate output layers.

\textbf{Experimental settings}. For this experiment, due to computational constraints, we deviate from the settings used in the main paper. Specifically, we here use $S=512$, $n_0=4$, and run four trials of 20 iterations. We use a RBF kernel.

\textbf{Problem}. The goal is to optimize the fine-tuning test accuracy of a pretrained ViT B16 model on the industrial WM811K dataset. The tunable parameters are: learning rate in $[1e-4,1e-1]$, momentum in $[0,1]$, weight decay in $[0,1]$, step size (for the learning rate scheduling) in $[1,100]$, and $\gamma \in [0,1]$.\footnote{This experiment conceptually builds on \citeauthor{Ament.2024}'s experiment on CNN training.}

\newpage
\section{Further results}
\label{sec:appendix_results}

\subsection{Variance reduction and ranking stability}\label{sec:appendix_var_stability_results}
We estimate the empirical macro variance of the acquisition estimator by fixing a BO state, drawing a fixed set of $n=64$ Sobol probe samples, rebuilding the acquisition independently $16$ times, and computing the sample variance across repeated estimator values at each probe point. We report the mean pointwise macro variance for the raw and orthogonalized estimators. (Note: this does \emph{not} affect the fitted surrogate model, which stays untouched). We repeat this at fixed points throughout training. Due to space reasons, we only report abridged results in the main text. We report results for (i) a Mat\'ern-$5/2$ kernel in section \ref{sec:appendix_matern_analysis}, (ii) a RBF kernel in section \ref{sec:appendix_rbf_analysis}, and (iii) a TPE surrogate in section \ref{sec:appendix_tpe_analysis}.

\subsubsection{Mat\'ern-$5/2$ kernel}\label{sec:appendix_matern_analysis}
In Table~\ref{table:variance_reduction_full}, we give the variance reduction of \framework for all benchmark functions. For higher-dimensional Michalewicz10 and Levy16, we found that larger $n_0$ were required
Across all benchmarks, we observe a strong reduction in variance, confirming that our proposed \framework stabilizes the acquisition estimate. Further, we report the ranking stabilities in Table~\ref{table:ranking_stability_full}, We observe that \framework again reduces the variance and consistently stabilizes the probe ranking.

\begin{table}[htbp]
    \centering
    \caption{Variance reduction by orthogonalization.}
    \label{table:variance_reduction_full}
\adjustbox{max width=\textwidth}{%
\begin{tabular}{rrllllllll}
\toprule
&  & \multicolumn{2}{c}{Hartmann6} & \multicolumn{2}{c}{Ackley8} & \multicolumn{2}{c}{Michalewicz10} & \multicolumn{2}{c}{Levy16}\\
$n_0$ & $S$ & $\operatorname{Var}(\widehat{\mathrm{EI}}_m^{\mathrm{MC}})$ & $\operatorname{Var}(\widehat{\mathrm{EI}}_m^{\mathrm{ortho}})$ & $\operatorname{Var}(\widehat{\mathrm{EI}}_m^{\mathrm{MC}})$ & $\operatorname{Var}(\widehat{\mathrm{EI}}_m^{\mathrm{ortho}})$ & $\operatorname{Var}(\widehat{\mathrm{EI}}_m^{\mathrm{MC}})$ & $\operatorname{Var}(\widehat{\mathrm{EI}}_m^{\mathrm{ortho}})$ & $\operatorname{Var}(\widehat{\mathrm{EI}}_m^{\mathrm{MC}})$ & $\operatorname{Var}(\widehat{\mathrm{EI}}_m^{\mathrm{ortho}})$ \\
\midrule
32 & 8 & 2e-08 & 1.61e-07 & 9.68e-09 & 5.24e-09 & 9.51e-08 & 4.54e-08 & 0.00106 & 0.000709 \\
32 & 32 & 2.59e-09 & 3.02e-10 & 6.3e-09 & 4.17e-10 & 1.67e-08 & 1.07e-09 & 0.000208 & 2.27e-05 \\
\bottomrule
\end{tabular}}
\end{table}

\begin{table}[h]
    \centering
    \caption{Ranking stability. Lower probe variance and flip rate indicate more stable acquisition estimates; higher ranking correlation and top-1 agreement indicate more stable candidate rankings across probes.}
    \label{table:ranking_stability_full}
\begin{tabular}{lllllll}
\toprule
Benchmark & Method & Probe variance & Top1 agr. & Flip rate & Regret \\
\midrule
\multirow{2}{*}{Hartmann6} & qLogEI & $8.416_{15.43}$ & $0.925_{0.16}$ & $0.114_{0.12}$ & $1.411_{0.42}$ \\
 & OrthoBO (ours) & $0.038_{0.04}$ & $0.950_{0.10}$ & $0.051_{0.04}$ & $1.232_{0.67}$ \\
\midrule
\multirow{2}{*}{Ackley8} & qLogEI & $19.471_{39.29}$ & $0.958_{0.11}$ & $0.102_{0.12}$ & $17.864_{2.08}$ \\
 & OrthoBO (ours) & $0.017_{0.04}$ & $0.986_{0.06}$ & $0.021_{0.05}$ & $18.136_{0.78}$ \\
\midrule
\multirow{2}{*}{Michalewicz10} & qLogEI & $168.603_{154.76}$ & $0.925_{0.16}$ & $0.148_{0.12}$ & $7.140_{0.01}$ \\
 & OrthoBO (ours) & $0.019_{0.03}$ & $0.988_{0.06}$ & $0.014_{0.03}$ & $6.961_{0.36}$ \\
\midrule
\multirow{2}{*}{Levy16} & qLogEI & $233.655_{287.67}$ & $0.917_{0.12}$ & $0.069_{0.09}$ &  $82.194_{6.05}$ \\
 & OrthoBO (ours) & $0.073_{0.21}$ & $0.988_{0.06}$ & $0.048_{0.04}$ &  $58.938_{0.00}$\\
\bottomrule
\end{tabular}
\end{table}

\subsubsection{RBF kernel}\label{sec:appendix_rbf_analysis}
We repeat our previous analysis, but use RBF kernel for the GP surrogate. We report the raw and orthogonalized variances in Table~\ref{table:appendix_variance_reduction_rbf}. Our observations are consistent: orthogonalization substantially reduces estimation variance.

\begin{table}[h]
    \centering
        \caption{Variance reduction through orthogonalization, using a RBF kernel.}
    \label{table:appendix_variance_reduction_rbf}
\adjustbox{width=\linewidth}{
\begin{tabular}{rrllllllll}
\toprule
&  & \multicolumn{2}{c}{Hartmann6} & \multicolumn{2}{c}{Ackley8} & \multicolumn{2}{c}{Michalewicz10} & \multicolumn{2}{c}{Levy16}\\
$n_0$ & $S$ & $\operatorname{Var}(\widehat{\mathrm{EI}}_m^{\mathrm{MC}})$ & $\operatorname{Var}(\widehat{\mathrm{EI}}_m^{\mathrm{ortho}})$ & $\operatorname{Var}(\widehat{\mathrm{EI}}_m^{\mathrm{MC}})$ & $\operatorname{Var}(\widehat{\mathrm{EI}}_m^{\mathrm{ortho}})$ & $\operatorname{Var}(\widehat{\mathrm{EI}}_m^{\mathrm{MC}})$ & $\operatorname{Var}(\widehat{\mathrm{EI}}_m^{\mathrm{ortho}})$ & $\operatorname{Var}(\widehat{\mathrm{EI}}_m^{\mathrm{MC}})$ & $\operatorname{Var}(\widehat{\mathrm{EI}}_m^{\mathrm{ortho}})$ \\
\midrule
32 & 8 & 5.91e-08 & 1.91e-07 & 1.28e-08 & 7.19e-09 & 1.15e-07 & 4.56e-08 & 0.00129 & 0.000758 \\
32 & 32 & 8.62e-09 & 7.3e-10 & 3.71e-09 & 4e-10 & 2.61e-08 & 1.76e-09 & 0.000301 & 2.85e-05 \\
\bottomrule
\end{tabular}}
\end{table}

\begin{table}[h]
    \centering
    \caption{Ranking stability, \textbf{RBF kernel}. Lower probe variance and flip rate indicate more stable acquisition estimates; higher ranking correlation and top-1 agreement indicate more stable candidate rankings across probes.}
    \label{tab:placeholder}
\begin{tabular}{lllllll}
\toprule
 Benchmark &  Method  &  Probe variance & Top1 agr. & Flip rate & Regret \\
\midrule
\multirow{2}{*}{Hartmann6} & qLogEI & $5.373_{11.64}$ & $0.944_{0.15}$ & $0.085_{0.10}$ & $0.126_{0.00}$ \\
 & OrthoBO (ours) & $0.031_{0.04}$ & $0.986_{0.08}$ & $0.031_{0.04}$ & $0.125_{0.00}$ \\
\midrule
\multirow{2}{*}{Ackley8} & qLogEI & $10.051_{31.34}$ & $0.986_{0.06}$ & $0.052_{0.10}$ & $12.476_{2.06}$ \\
 & OrthoBO (ours) & $0.013_{0.04}$ & $0.993_{0.04}$ & $0.019_{0.04}$ & $12.581_{4.03}$ \\
\midrule
\multirow{2}{*}{Michalewicz10} & qLogEI & $55.138_{100.54}$ & $0.931_{0.15}$ & $0.085_{0.09}$ & $4.517_{0.20}$ \\
 & OrthoBO (ours) & $0.002_{0.01}$ & $1.000_{0.00}$ & $0.031_{0.04}$ & $4.004_{0.52}$ \\
\midrule
\multirow{2}{*}{Levy16} & qLogEI & $82.913_{154.71}$ & $0.889_{0.18}$ & $0.201_{0.10}$ & $3.109_{0.57}$ \\
 & OrthoBO (ours) & $0.034_{0.17}$ & $1.000_{0.00}$ & $0.049_{0.04}$ & $2.819_{0.57}$ \\
\bottomrule
\end{tabular}
\end{table}

\subsubsection{TPE-based surrogate}\label{sec:appendix_tpe_analysis}
In Table~\ref{table:appendix_tpe_variance_reduction}, we show the variance reduction from \framework for a TPE surrogate, using a MC budget of $S=32$. Our observations are in line with the GP surrogate: \framework can substantially reduce the variance.

\begin{table}[h]
    \centering
    \tiny
    \caption{Variance reduction by orthogonalization, for a TPE surrogate.}
    \label{table:appendix_tpe_variance_reduction}
    \adjustbox{max width=\textwidth}{%
    \begin{tabular}{rrrrrrrr}
\toprule
 \multicolumn{2}{c}{Hartmann6} & \multicolumn{2}{c}{Ackley8} & \multicolumn{2}{c}{Michalewicz10} & \multicolumn{2}{c}{Levy16}\\
$\operatorname{Var}(\mathrm{EI}^{\mathrm{TPE}})$ & $\operatorname{Var}(\widehat{\mathrm{EI}}^{\mathrm{TPE,orth}})$ & $\operatorname{Var}(\mathrm{EI}^{\mathrm{TPE}})$ & $\operatorname{Var}(\widehat{\mathrm{EI}}^{\mathrm{TPE,orth}})$ & $\operatorname{Var}(\mathrm{EI}^{\mathrm{TPE}})$ & $\operatorname{Var}(\widehat{\mathrm{EI}}^{\mathrm{TPE,orth}})$ & $\operatorname{Var}(\mathrm{EI}^{\mathrm{TPE}})$ & $\operatorname{Var}(\widehat{\mathrm{EI}}^{\mathrm{TPE,orth}})$ \\
\midrule
$11.222_{16.405}$ & $1.214_{2.620}$ & $13.182_{18.233}$ & $1.284_{2.601}$ & $21.712_{25.982}$ & $2.197_{3.358}$ & $20.944_{26.200}$ & $1.399_{2.259}$ \\
\bottomrule
\end{tabular}}
\end{table}

\subsection{Misspecified surrogate families}
We first study robustness to surrogate misspecification. In practical BO, the surrogate family is only an approximation to the unknown objective, and kernel hyperparameters are often estimated from limited data. As a result, the posterior quantities used inside MC acquisition estimates can be noisy or biased, which may destabilize candidate rankings. This experiment evaluates whether the variance reduction induced by \framework improves BO decisions under such challenging, but realistic, surrogate conditions.

We consider four standard synthetic regression problems with known global optima and varying dimensionality: Hartmann6, Ackley8, Michalewicz10, and Levy16 \citep[e.g.,][]{Ament.2023,Bodin.2020,eriksson2021scalable,moss2023inducing}. Details of the benchmark functions and surrogate choices are provided in Table~\ref{tab:function_wrong_kernels}. We compare methods using best-so-far regret.

\textbf{Mild misspecification: isotropic RBF kernel.}
We first use an isotropic RBF kernel, a common smooth GP prior. For anisotropic and multimodal objectives it is mildly misspecified because it uses a single lengthscale across all dimensions. (The dimension of the benchmark functions varies from $6$ to $16$.) The results are in Fig.~\ref{fig:kernel_rbf}.

\begin{figure}[h]
    \centering
    \begin{subfigure}{0.245\textwidth}
    \centering
    \includegraphics[width =\textwidth]{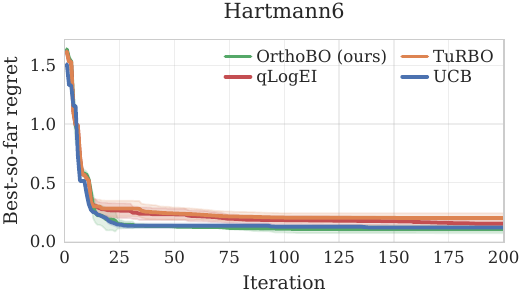}
    \end{subfigure}
    \begin{subfigure}{0.245\textwidth}
    \centering
    \includegraphics[width = \textwidth]{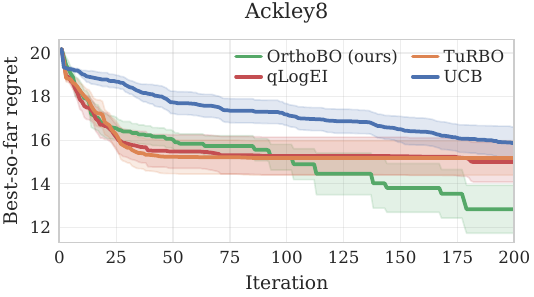}
    \end{subfigure}
    \begin{subfigure}{0.245\textwidth}
    \centering
    \includegraphics[width = \textwidth]{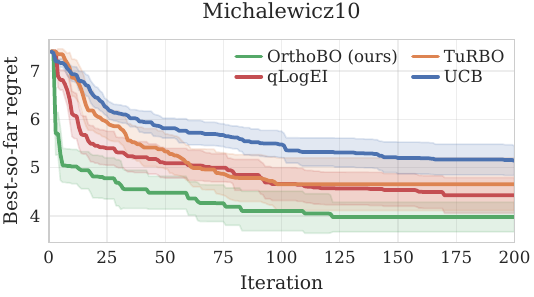}
    \end{subfigure}
    \begin{subfigure}{0.245\textwidth}
    \centering
    \includegraphics[width = \textwidth]{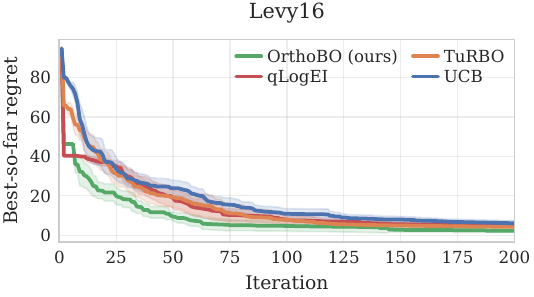}
    \end{subfigure}
\vspace{-0.5cm}
\caption{Best-so-far regret as a function of iterations on four widely used regression problems, using a \textbf{mildly misspecified isotropic RBF kernel}. Our proposed \framework is competitive, and in higher-dimensional functions (Michalewicz10, Levy16) outperforms the baselines.}
\label{fig:kernel_rbf}
\end{figure}

\textbf{Results.} Under  mild surrogate misspecification, our proposed \framework remains competitive across all benchmarks. On Hartmann6 and Ackley8, it improves over qLogEI and reaches the lowest regret. On the higher-dimensional Michalewicz10 and Levy16 functions, \framework achieves the lowest regret and improves substantially earlier than the baselines. For example, \framework reaches a regret of $\sim5$ after 5 iterations, whereas the baseline methods require at least 50 to 60 \emph{additional iterations}. Together, these results indicate that variance reduction at the acquisition-estimation level improves the reliability of BO decisions under realistic surrogate mismatch, with the largest gains appearing in higher-dimensional settings where candidate rankings are more sensitive to estimation noise.

\textbf{Strong misspecification: linear kernel.}
We next consider a more challenging setting by using a linear kernel. This surrogate is unable to represent the nonlinear and multimodal structure of the benchmark functions, which amplifies the effect of posterior and acquisition-estimation errors. We therefore test whether stabilizing the acquisition estimate can still improve BO decisions when the surrogate family is substantially imperfect. The results are shown in Fig.~\ref{fig:kernel_linear}.

\begin{figure}[h]
    \centering
    \begin{subfigure}{0.24\textwidth}
    \centering
    \includegraphics[width = \textwidth]{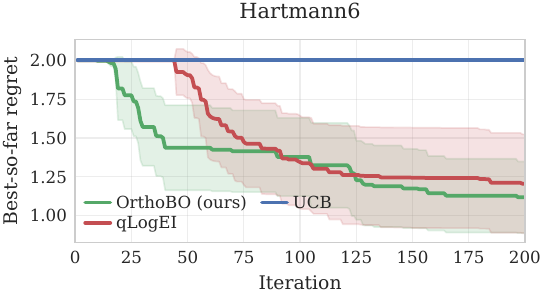}
    \end{subfigure}
    \begin{subfigure}{0.24\textwidth}
    \centering
    \includegraphics[width = \textwidth]{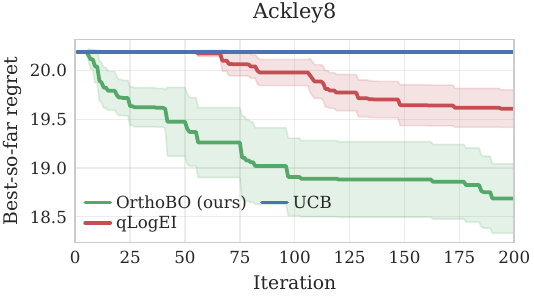}
    \end{subfigure}
    \begin{subfigure}{0.24\textwidth}
    \centering
    \includegraphics[width = \textwidth]{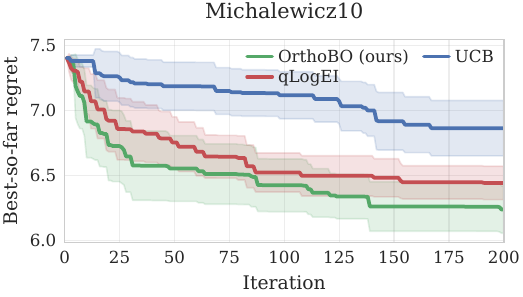}
    \end{subfigure}
    \begin{subfigure}{0.24\textwidth}
    \centering
    \includegraphics[width = \textwidth]{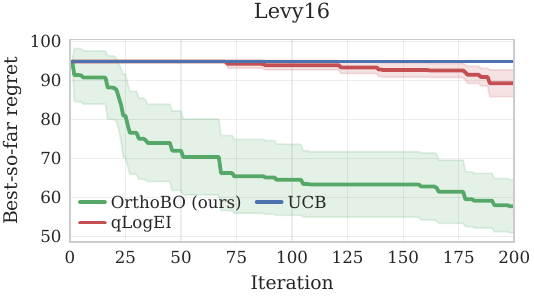}
    \end{subfigure}
\caption{Best-so-far regret as a function of iterations on four widely used regression problems, using a \textbf{strongly misspecified linear kernel}. Our proposed \framework substantially outperforms other methods. The performance difference increases as the problem dimensionality grows. Notably, on Hartmann6, qLogEI takes five times more iterations than \framework to reach the same level of regret.}
\label{fig:kernel_linear}
\end{figure}

\textbf{Results.} We observe: (i) On all regression problems, our \framework outperforms the baselines. (ii) on Hartmann6, \framework reaches the same regret as qLogEI in half the time. (iii) On Levy16, \framework is the only method to substantially improve over time. Together, the results demonstrate the orthogonalization can successfully overcome strong misspecifications of the surrogate family.

With the strongly misspecified linear kernel, \framework outperforms the baselines on all four functions. On Hartmann6, it reaches the same regret level as qLogEI with substantially fewer evaluations. On Levy16, \framework is the only method that shows a clear improvement over the optimization horizon. Together, these results indicate that acquisition-estimation noise can be a significant failure mode when the surrogate is imperfect, and that orthogonalization helps by stabilizing the acquisition values used to rank candidates. Thus, \framework does not require a perfectly specified surrogate to improve BO performance; rather, it provides a variance-reduction mechanism that is especially useful when MC acquisition estimates become unreliable.

\subsection{Weakly fitted hyperparameters}
We next study a data-scarce regime in which the surrogate family is expressive, but its hyperparameters are weakly identified. This situation is common in expensive HPO problems: only a small initial design is available, yet the GP must already estimate lengthscales, output scale, and noise parameters that determine the posterior and hence the acquisition values. Uncertainty or instability in these hyperparameters can therefore propagate to the acquisition function and distort candidate rankings.
\vspace{-0.2cm}
\begin{wrapfigure}[16]{r}{0.5\linewidth}
    \centering
    \includegraphics[width=\linewidth]{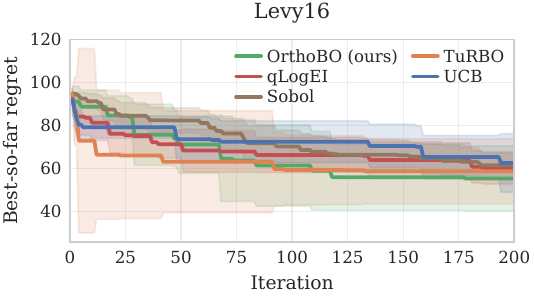}
    \caption{Best-so-far regret on Levy16 with a Mat\'ern-$5/2$ ARD kernel and only $n_0=4$ initial Sobol points. \framework remains among the strongest methods and improves steadily despite weakly identified surrogate hyperparameters.}
    \label{fig:weak_fit_levy16}
\end{wrapfigure}

To isolate this effect, we use a GP with a Mat\'ern-$5/2$ kernel and ARD. Unlike the isotropic kernels studied above, ARD assigns a separate lengthscale to each input dimension. This makes the surrogate more flexible, but also harder to fit when the initial design is small. We therefore initialize BO with only $n_0=4$ Sobol points and evaluate performance on Levy16, which is particularly challenging in this setting because the surrogate must estimate 16 lengthscales and a noise parameter from very limited data. The results are shown in Fig.~\ref{fig:weak_fit_levy16}.

\textbf{Results.} All methods initially exhibit high regret, reflecting the difficulty of fitting an ARD surrogate from only four initial observations. Nevertheless, \framework decreases regret steadily over the optimization horizon and reaches one of the lowest final regrets among the compared methods. These results show that orthogonalization remains useful even when the kernel class is flexible but its hyperparameters are poorly estimated. The gains are consistent with the proposed mechanism: reducing sensitivity of the acquisition estimate to surrogate-posterior perturbations yields more reliable candidate rankings in early, data-limited BO.

\subsection{Ensembling}\label{sec:appendix_ensemble}
In the main experiments, we used $M=1$ to avoid confounding the effect of orthogonalization with mere performance gains from a larger ensemble. We now use an ensemble of $M=3$ models, each with a different kernel: linear, RBF, and Mat\'ern-$5/2$. We increased the number of iterations by \SI{25}{\percent} to account for additional fitting iterations for the models. We compute $\ell_{m,t}$ from the GP surrogate posteriors, set $\tau = 1$, and use $\delta=0.001$ as a minimum weight floor.

The results for ensembling are in Figure~\ref{fig:ensembling}. 

Further, we compute (i) the variance reduction, (ii) ranking stability metrics, and (iii) weighting metrics. These additional measures are computed on a Sobol probe. We present the results in Table~\ref{tab:appendix_ensemble_variance} and Figure~\ref{fig:appendix_ensembling_entropy}.

\begin{figure}[h]
    \centering
    \includegraphics[width=0.5\linewidth]{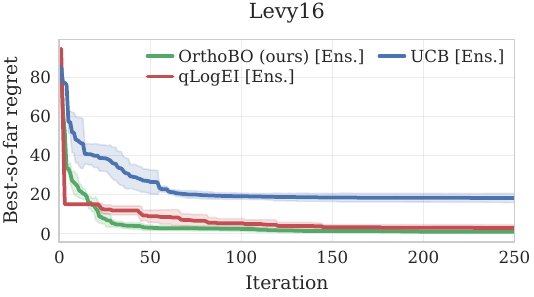}
    \caption{Results for ensembling.}
    \label{fig:ensembling}
\end{figure}

\begin{table}[h]
    \centering
    \caption{Variance reduction for ensembling on Levy16.}
    \label{tab:appendix_ensemble_variance}
\adjustbox{max width=\textwidth}{%
\begin{tabular}{lllllllll}
\toprule
Method & Probe variance & Top1 agr. & Flip rate & Regret & Entropy & Weight 1 & Weight 2 & Weight 3 \\
\midrule
UCB [Ens.] & $39.255_{48.64}$ & $1.000_{0.00}$ & $0.026_{0.04}$ & $18.267_{2.40}$ & $0.163_{0.18}$ & $0.008_{0.06}$ & $0.062_{0.14}$ & $0.933_{0.14}$ \\
qLogEI [Ens.]  & $390.598_{1033.83}$ & $0.875_{0.20}$ & $0.167_{0.10}$ & $2.926_{1.66}$ & $0.145_{0.16}$ & $0.014_{0.07}$ & $0.073_{0.21}$ & $0.917_{0.21}$ \\
\framework (Ours) [Ens.] & $0.039_{0.12}$ & $0.958_{0.14}$ & $0.081_{0.08}$ & $0.941_{0.54}$ & $0.169_{0.23}$ & $0.008_{0.06}$ & $0.069_{0.14}$ & $0.927_{0.14}$ \\

\bottomrule
\end{tabular}}
\end{table}

\textbf{Results}. We observe: (1) \framework achieves the lowest regret. (iii) \framework has substantially lower Sobol probing variance. (iii) Weight are highest for the Mat\'ern-$5/2$ kernel (``Weight 3'') demonstrating that all methods primarily focus the ``most correct'' kernel. We plot the evolution of the entropy over time in Figure~\ref{fig:appendix_ensembling_entropy}, which shows that the weight aggregation scheme balances exploration and exploitation.

\begin{figure}[h]
    \centering
    \includegraphics[width=0.5\linewidth]{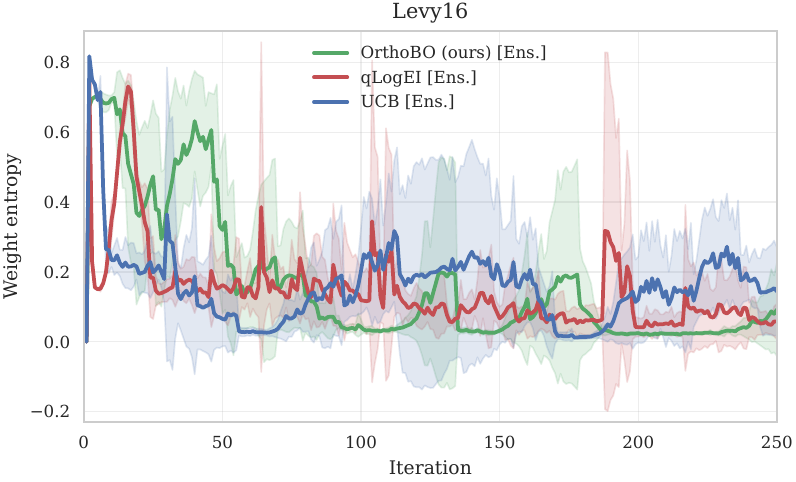}
    \caption{Evolution of the entropy across ensembling weights. For all methods, the used tempered exponentially weighted aggregation scheme is effective in balancing exploration and exploitation at the \emph{ensemble} level.}
    \label{fig:appendix_ensembling_entropy}
\end{figure}

\newpage
\subsection{Monte Carlo sampling ablation experiment}
We here study the effect of the MC budget used to estimate the acquisition function. Here, finite-sample acquisition noise can change candidate rankings, with the effect becoming more pronounced for small sampling budgets $S$ \citep[cf.][]{Ament.2023}. This setting is practically relevant because acquisition functions are evaluated many times during optimization, so increasing $S$ can be computationally expensive. We therefore test whether the variance reduction of \framework improves BO performance when acquisition estimates must be computed from limited MC samples.

We use an isotropic RBF kernel, thereby combining moderate surrogate mismatch with reduced acquisition-estimation budgets. We vary $S \in \{64,128,256,512\}$. The main results for Ackley8 and Michalewicz10 are shown in Fig.~\ref{fig:mc_ablation}; additional results are provided in Fig.~\ref{fig:mc_ablation_full}.

\begin{figure}[htbp]
    \centering
    \includegraphics[width=\textwidth]{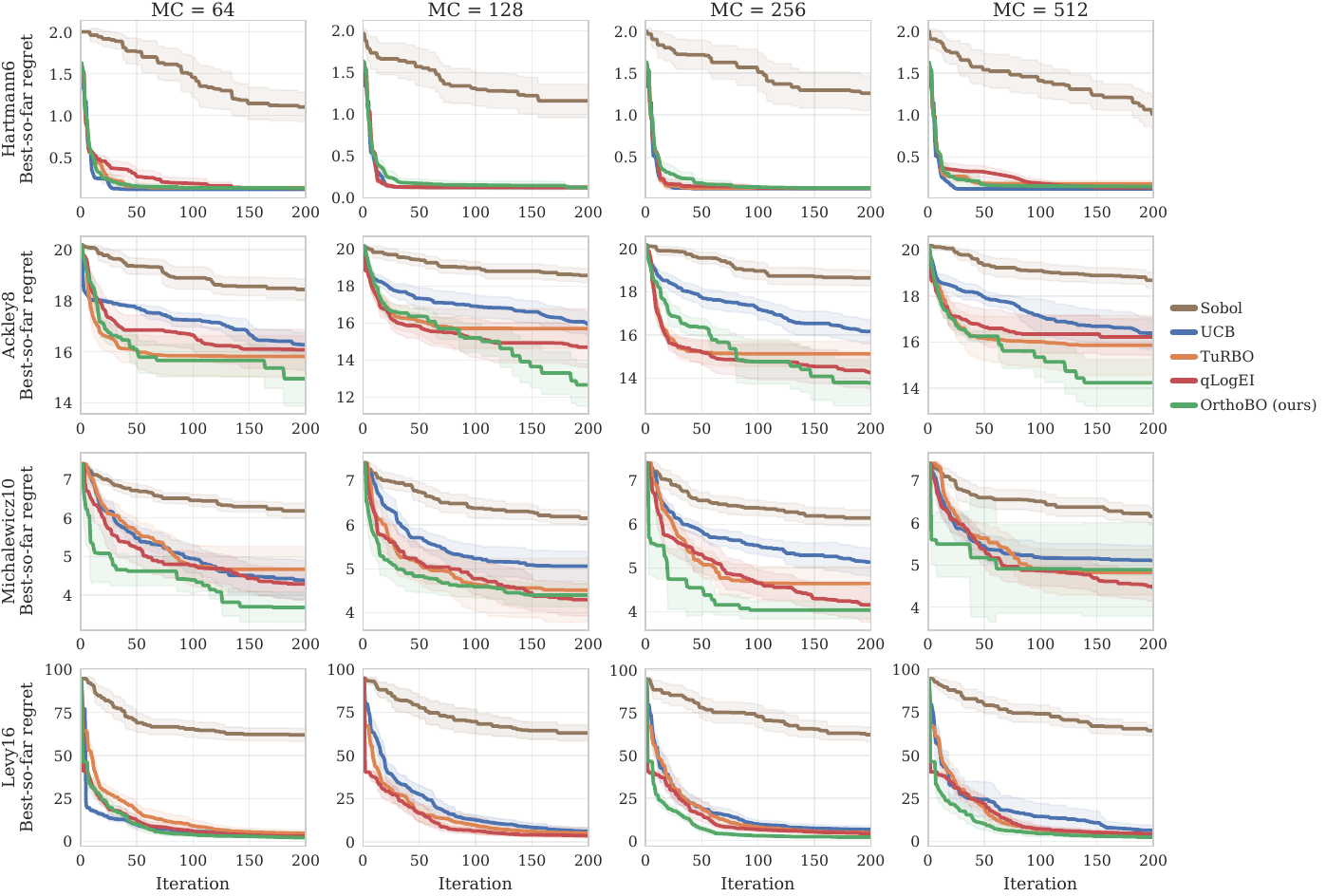}
    \caption{Different mc samples}  
    \label{fig:mc_ablation_full}
\end{figure}

\textbf{Results} Reducing the MC budget makes the performance of acquisition-based methods more sensitive to estimation noise. Across the tested budgets, \framework remains competitive on Ackley8 and consistently achieves low regret on Michalewicz10. The effect of orthogonalization are most visible at earlier iterations, where MC estimation noise has a larger effect on the acquisition values and hence on the induced candidate rankings. Overall, these results support the central mechanism of \framework: orthogonalization improves the reliability of acquisition estimates when MC budgets are limited, leading to more stable BO decisions and thus lower regret. 

\subsection{CIFAR10}
We again utilize the outlier experiment by \citet{Ament.2024}. However, this time we utilize the more challenging CIFAR10 dataset on a modified CNN architecture. To model corrupted evaluations, we prematurely stop training with probability \SI{20}{\percent} after observing between 100 and 1000 training samples. We optimize the same parameters as \citet{Ament.2024}. To not implicitly mitigate the corrupted evaluations, we only use $n_0=2$ initial points. The results are in Figure~\ref{fig:appendix_cifar10}. \framework achieves the best performance, despite 20\% corruptions during training.

\begin{figure}
    \centering
    \includegraphics[width=0.75\linewidth]{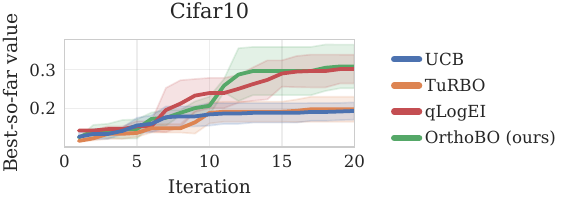}
    \caption{Results for HPO optimization for training on CIFAR10 from scratch.}
    \label{fig:appendix_cifar10}
\end{figure}

\end{document}